\documentclass{article}

\usepackage{microtype}
\usepackage{graphicx}
\usepackage{subcaption}
\usepackage{booktabs} % for professional tables

\usepackage{hyperref}

\usepackage[accepted]{icml2026}

\usepackage{amsmath}
\usepackage{amssymb}
\usepackage{mathtools}
\usepackage{amsthm}
\usepackage{stfloats}
\usepackage{pifont}
\usepackage{xcolor}
\usepackage{tikz}
\usepackage{pgfplots}

\usepackage[capitalize,noabbrev]{cleveref}

\theoremstyle{plain}
\newtheorem{theorem}{Theorem}[section]

\theoremstyle{definition}
\newtheorem{definition}[theorem]{Definition}

\theoremstyle{remark}

\usepackage[textsize=tiny]{todonotes}

\icmltitlerunning{\textsc{Memora}: A Harmonic Memory Representation Balancing Abstraction and Specificity}

\begin{document}

\twocolumn[
  \icmltitle{\textsc{Memora}: A Harmonic Memory Representation \\Balancing Abstraction and Specificity}



  \icmlsetsymbol{equal}{*}

  \begin{icmlauthorlist}
    \icmlauthor{Menglin Xia}{equal,M365}
    \icmlauthor{Xuchao Zhang}{equal,M365}
    \icmlauthor{Shantanu Dixit}{M365}
    \icmlauthor{Paramaguru Harimurugan}{M365}
    \icmlauthor{Rujia Wang}{M365}
    \icmlauthor{Victor R{\"u}hle}{M365}
    \icmlauthor{Robert Sim}{M365}
    \icmlauthor{Chetan Bansal}{M365}
    \icmlauthor{Saravan Rajmohan}{M365}
  \end{icmlauthorlist}

  \icmlaffiliation{M365}{M365 Research, Microsoft}

  \icmlcorrespondingauthor{Menglin Xia}{mollyxia@microsoft.com}
  \icmlcorrespondingauthor{Xuchao Zhang}{xuchaozhang@microsoft.com}
  \icmlkeywords{Machine Learning, ICML}

  \vskip 0.3in
]



\printAffiliationsAndNotice{\icmlEqualContribution}

\begin{abstract}
% Agent memory systems must manage continuously growing information while supporting efficient, context-aware retrieval for downstream tasks. Abstraction is essential for organizing memory—consolidating related information into stable, maintainable structures that scale with memory growth. However, overly coarse abstraction can obscure fine-grained, task-specific context, making it difficult for agents to retrieve the precise information needed for effective reasoning. We introduce \textbf{Memora}, a h\textbf{ar}m\textbf{o}nic \textbf{mem}ory representation that balances abstraction and specificity through complementary memory structures. Memora organizes memories using abstraction-level representations that provide stable, well-structured memory units, while cue-based associations preserve fine-grained and contextual access paths across memories. Retrieval is governed by their coordinated interaction rather than a single similarity metric, enabling agents to efficiently identify relevant memory subsets without exhaustive expansion. We also demonstrate that retrieval-augmented generation and knowledge-graph-based memory emerge as special cases of Memora. Experiments on agent reasoning tasks demonstrate improved retrieval relevance and reasoning efficiency as memory scales. (TODO: Add detailed improvement on xxx datasets)

Agent memory systems must accommodate continuously growing information while supporting efficient, context-aware retrieval for downstream tasks. Abstraction is essential for scaling agent memory, yet it often comes at the cost of specificity, obscuring the fine-grained details required for effective reasoning. We introduce \textsc{Memora}, a harmonic memory representation that structurally balances abstraction and specificity. \textsc{Memora} organizes information via its \textit{primary abstractions} that index concrete memory values and consolidate related updates into unified memory entries, while \textit{cue anchors} expand retrieval access across diverse aspects of the memory and connect related memories. Building on this structure, we employ a retrieval policy that actively exploits these memory connections to retrieve relevant information beyond direct semantic similarity. Theoretically, we show that standard Retrieval-Augmented Generation (RAG) and Knowledge Graph (KG)-based memory systems emerge as special cases of our framework. Empirically, \textsc{Memora} establishes a new state-of-the-art on the LoCoMo and LongMemEval benchmarks, demonstrating better retrieval relevance and reasoning effectiveness as memory scales.

\end{abstract}

\section{Introduction}

Large language models (LLMs) have substantially advanced the capabilities of autonomous agents in planning, tool use, and multi-step reasoning \cite{Wang_2024, guo2024largelanguagemodelbased}. However, intelligence is not just the ability to reason in the moment; it is the ability to learn and adapt over time--a capability rooted in how experience is organized, abstracted, and reused. While current agents excel at atomic problem-solving, they remain effectively \textit{stateless}, treating recurring tasks and user intents as isolated events \cite{yao2023react, wu2023autogen}. Without a principled mechanism to organize accumulated experience, agents are forced to repeatedly re-derive plans and reproduce redundant reasoning steps, leading to brittle performance and escalating token costs. As agents are increasingly deployed in real-world environments, this lack of structured, reusable memory has become the critical bottleneck, limiting their ability to support complex, long-horizon workflows \cite{Milam2025context}.

Scaling agent memory requires resolving a fundamental tension between abstraction and specificity. Existing designs typically collapse into one of two extremes. Many approaches favor specificity, either by storing raw interactions or document fragments \cite{xu2025amemagenticmemoryllm, lewis2021retrieval} or by extracting atomic facts from text \cite{chhikara2025mem0, nan2025nemori}. While detailed, these strategies suffer from fragmentation: raw logs overwhelm the agent with unstructured noise, while isolated facts stripped of their narrative context often fail to capture dependencies inherent in long-horizon tasks. Conversely, others adopt coarse abstractions, compressing experience into high-level summaries \cite{zhong2023memorybank, li2025memos}. While efficient, this approach strips away task-critical nuances (e.g., specific constraints, edge cases, or numeric details), rendering the memory insufficient for precise execution. This representational gap cripples retrieval: because memory lacks a structured link between high-level concepts and low-level details, agents cannot effectively navigate their own history. They are left choosing between retrieving a deluge of irrelevant facts or a vague summary that lacks actionable utility, ultimately failing to support robust long-horizon reasoning.

\begin{figure*}[t]
    \centering
    \includegraphics[width=\textwidth]{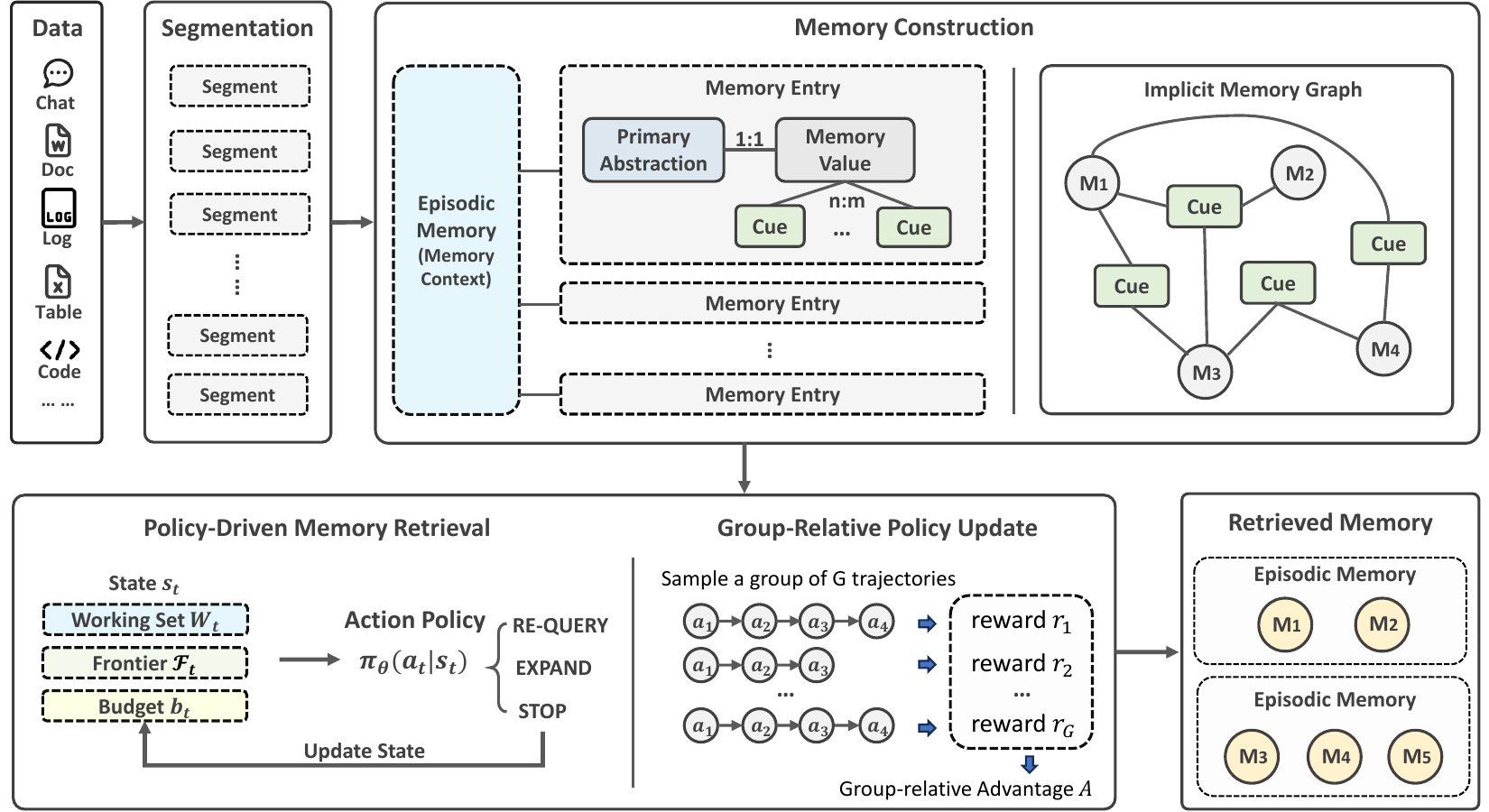}
    \caption{Overview of the \textsc{Memora} harmonic memory architecture.}
    \label{fig:overview}
\end{figure*}

To address these limitations, we introduce \textbf{\textsc{Memora}}, a harmonic memory architecture that structurally balances abstraction and specificity. \textsc{Memora} organizes experience through a dual-layered representation that acts as navigational scaffolding over concrete content. At the core is the \textit{primary abstraction}, which defines the canonical identity of a memory entry --- capturing what the memory is fundamentally about. Each memory entry is composed of a primary abstraction paired with a memory value, where the value stores the specific memorized information. The primary abstraction acts as a coherent container, enabling \textsc{Memora} to incorporate emerging concepts as new entries while aggregating related updates into a unified record, thereby preventing conceptually related information from fragmenting into disjoint memory entries. For example, the evolving timeline of a project can be represented as a single memory entry under the primary abstraction \textit{Project Memora Timeline}, within which milestones, design iterations, experiments, and decisions are incrementally appended. Complementing this, \textit{cue anchors} are extracted from the memory value to serve as contextualized access points. By encoding diverse perspectives and aspects of a memory, these anchors expand retrieval access and establish a many-to-many connectivity across related memory entries. Together, this organization allows agents to navigate from concrete contexts to stable abstractions, supporting implicit relational reasoning and temporal coherence without the overhead of full-context processing.

Furthermore, we introduce a policy-guided retrieval mechanism that treats memory access as an active reasoning process. Retrieval is formulated over a discrete action space consisting of query refinement, memory expansion, and termination. By iteratively selecting these actions, the policy retriever refines the retrieved context to uncover relevant information beyond immediate semantic similarity, effectively capturing multi-hop dependencies that static retrieval methods often miss.

Empirically, \textsc{Memora} establishes state-of-the-art performance on the LoCoMo and LongMemEval benchmarks (86.3\% and 87.4\% respectively), outperforming both strong memory baselines and full-context inference. Its ability to consistently outperform full-context inference demonstrates that memory retrieval guided by appropriate abstraction is more reliable than brute-force reconstruction for reasoning over extensive histories. By balancing abstraction with specificity, the harmonic organization of \textsc{Memora} provides a scalable foundation for long-horizon agent intelligence, reducing token consumption by up to 98\% compared to full-context processing. The code is available at \url{https://github.com/microsoft/Memora}.

\section{Related Work}

\paragraph{Agentic Memory Management Systems} Retrieval-Augmented Generation (RAG) \citep{lewis2021retrieval, borgeaud2022improving, gao2024retrieval} effectively extends the context capabilities of LLMs, but often lacks the precision required for long-horizon reasoning for agentic tasks. Consequently, recent research has shifted toward active memory management. Systems like MemGPT \citep{packer2023memgpt} draw inspiration from operating systems, introducing a virtual context management system that actively swaps information between ``active" context and archival storage. Similarly, MemOS \citep{chhikara2025mem0}, Memory OS \citep{kang2025memoryosaiagent} and MIRIX \cite{wang2025mirix} propose architecture-level solutions for managing memory lifecycles. Other approaches focus on the mechanism of interaction: LangMem\footnote{\url{https://langchain-ai.github.io/langmem/}} treats memory as an external tool that agents explicitly call to update, while learning-based approaches like Mem-R1 \citep{yan2026memoryr1} attempt to train models to manage their own memory policies autonomously.

\paragraph{Structured Memory Representations} Parallel to management strategies, significant research has focused on how memory is represented and structured to improve organization and retrieval. Early attempts like MemoryBank \citep{zhong2023memorybank} utilized summarization to condense past events, while A-Mem \citep{xu2025amemagenticmemoryllm} grouped memories into clusters. Mem0 \citep{chhikara2025mem0} takes a different approach, prioritizing the lifecycle of factual memories with explicit mechanisms to add, update, and delete extracted facts. Nemori \cite{nan2025nemori} attempts to combine episodic and semantic memory types to mirror human cognitive processes. However, without a cohesive structure, these isolated facts often become fragmented, leading to significant information loss during updates. Concurrently, graph-based representations, such as HippoRAG \citep{gutierrez2024hipporag}, GraphRAG \citep{edge2025localglobalgraphrag}, Zep \citep{rasmussen2025zep}, and Mem0-graph \citep{chhikara2025mem0}, have emerged to capture relationships between entities and support global reasoning. While graphs improve connectivity, they introduce distinct trade-offs: rigid schemas often abstract away critical details, while maintaining dense graph structures at scale can introduce significant retrieval noise. In addition, despite the structural innovations, the underlying representations often remain brittle, struggling to balance the specificity required for precision with the abstraction needed for scalability.

\paragraph{Positioning of \textsc{Memora}} Existing memory systems often index and retrieve memories through their \emph{content}, either as raw embeddings of extracted facts or via graph traversal over entity nodes. This couples \emph{what} is stored to \emph{how} it is accessed and forces the specificity--abstraction trade-off discussed above. \textsc{Memora}'s central novelty is to decouple these two concepts: memory content remains rich and unfragmented, while a separate structural layer of primary abstractions and cue anchors carries the navigational signal used for retrieval. Unlike graph-based methods, \textsc{Memora} also avoids the maintenance overhead of a pre-defined ontology: cue anchors are generated per memory and double as flexible metadata filters (source, timestamp, entity), so the same memory is reachable through multiple alternative access paths without committing to a fixed schema.

We provide a formal analysis demonstrating that \textsc{Memora} serves as a unified and strictly more expressive framework for memory retrieval. Traditional RAG and KG-based retrieval emerge as special cases under restricted configurations, while \textsc{Memora} supports richer mixed-key retrieval behaviors and principled efficiency improvements through abstraction-first scoping and structured traversal. More details including the proof can be found in Appendix \ref{appendix:theory}.

\section{Method}

We propose \textsc{Memora}, a harmonic memory representation designed to balance abstraction with specificity. We begin by formalizing the problem setting, followed by a detailed description of the proposed method.

\subsection{Problem Formulation}

We formulate memory management as the maintenance of a structured store derived from a continuous, heterogeneous data stream.

Let $\mathcal{D} = \{d_1, \ldots, d_N\}$ denote a growing corpus of documents, logs, code, tables, or agentic interaction traces.

Our objective is to learn a memory construction function
\[
\mathcal{F}_m : \mathcal{D} \rightarrow \mathcal{M},
\]
that maps raw data to a structured memory set $\mathcal{M}$, and a retrieval function

\[
\mathcal{Q}(q, \mathcal{M}) \rightarrow \mathcal{M}_q, \quad \mathcal{M}_q \subseteq \mathcal{M},
\]

that, given a query $q$, selects a compact subset of relevant memory entries $\mathcal{M}_q$ to maximize downstream task utility.

The core design challenge is to maximize the relevance of $\mathcal{M}_q$ while minimizing its size ($|\mathcal{M}_q| \ll |\mathcal{M}|$) and retrieval latency, necessitating a representation that supports both high-level semantic scanning and fine-grained contextual lookup.
%By decoupling stable abstraction from contextual access, \textsc{Memora} enables scalable memory retrieval that supports long-horizon reasoning without reliance on brute-force long-context inference.

\subsection{\textsc{Memora} Overview}

Figure~\ref{fig:overview} illustrates the overall architecture of \textsc{Memora}. Raw data from multiple sources are first segmented into semantic units, each associated with episodic context capturing situational information. These segments are transformed into harmonic memory entries, where each entry consists of a primary abstraction paired with a memory value and augmented with cue anchors. Primary abstractions provide stable canonical identities that consolidate related and evolving information, while cue anchors induce many-to-many associations across memory entries. Based on shared cue anchors and abstraction-level relationships, these associations give rise to an implicit memory graph that encodes relational structure among memory entries without requiring explicit edge construction. At query time, an agent query is jointly matched against primary abstractions and cue anchors to identify relevant memory entries. Memory reasoning then traverses the resulting abstraction- and cue-based associations to retrieve a coherent set of related memory entries together with their episodic contexts. This design enables scalable, context-aware retrieval that supports downstream reasoning, planning, and decision-making without requiring full interaction histories to be reconstructed in the context window. The retrieval policy can be further optimized using Group-Relative Policy Optimization, which trains the policy by comparing groups of retrieval trajectories and updating it based on relative advantages, encouraging effective multi-step navigation and early stopping behavior.

\subsection{Segmentation}
Given a data item $d \in \mathcal{D}$, we first apply a segmentation function $\mathcal{S}(d)$ to decompose the content $x$ into a set of semantically coherent segments $\{s_1, \ldots, s_k \}$.
Each segment $s_i$ serves as the input unit for memory construction. This segmentation step determines the granularity at which memory entries are created and updated, enabling primary abstractions to consolidate related information while preserving contextual specificity. Notably, a single segment may give rise to multiple memory entries. The implementation of $\mathcal{S}$ depends on the data format: we employ a prompt-based extraction mechanism for unstructured narratives, but leverage structural hierarchies (such as document headers) for formatted files.
\subsection{Episodic Memory}

Episodic memory in \textsc{Memora} captures the narrative context associated with each segment. For every segment $s_i$, we construct an episodic memory $e_i = \mathcal{E}(s_i)$ that serves as a shared narrative grounding for all memory entries derived from that source. Crucially, the representation of $e_i$ is flexible: it can take the form of an extracted high-level summary--capturing participants, intent, and temporal scope--or retain the raw segment text itself to preserve exact phrasing and subtle cues. This design allows episodic memory to function as a contextual anchor, adapting the balance between compression and fidelity based on the domain.

% Rather than encoding fine-grained factual details, episodic memory captures high-level information such as the situation, participants, and temporal scope of the interaction. This design enables episodic memory to function as a contextual abstraction that complements more detailed memory representations. The episodic memory construction process is implemented via a prompt-driven extraction mechanism.

During memory retrieval and reasoning, episodic memories play a central role in preserving narrative coherence across retrieved items. Memory entries associated with the same episodic memory are grouped together, allowing the agent to recover the broader context surrounding individual facts. This episodic grouping supports coherent multi-step reasoning, planning, and decision-making in downstream agent workflows.

% \subsection{Primary Abstraction}

% Each memory entry $m \in \mathcal{M}$ is associated with a \emph{primary abstraction} $a(m) \in \mathcal{A}$, which represents the core concept or action distilled from accumulated experience. The primary abstraction serves as the stable organizing reference for the memory entry, determining how related information is grouped, consolidated, and maintained over time.

% Rather than encoding a single observation, the primary abstraction enables multiple related segments to be organized under the same memory entry when they describe the same underlying concept, even if they differ in surface form, context, or temporal occurrence. As new information arrives, content that aligns with an existing primary abstraction is incrementally incorporated into the corresponding memory entry, allowing the memory to evolve without fragmenting into redundant records.

% By operating at a coarse but semantically meaningful level, primary abstractions provide a compact structure over the memory space, reducing redundancy and supporting efficient retrieval. At the same time, they avoid over-specification, ensuring that abstraction does not obscure task-relevant detail, which is preserved through complementary mechanisms such as cue anchors. Together, this design enables Memora to maintain coherent, long-lived memory entries that support consistent reasoning across interactions.

\subsection{Primary Abstraction}

% The primary abstraction $a_j$ canonically represents the core concept or action distilled from one or more data items, serving as the stable organizing unit of memory. Multiple related observations, updates, or contextual details may be consolidated under the same primary abstraction, enabling the memory entry to evolve over time rather than fragment across redundant records. The memory content $u_j$ aggregates the associated representations, including values, attributes, temporal information, and contextual signals that ground the abstraction in concrete experience. To support contextualized access, \textsc{Memora} associates each memory entry with a set of \emph{cue anchors}.

To prevent memory fragmentation, we introduce \emph{primary abstraction} to organize memory around stable, semantically meaningful concepts rather than individual observations. A primary abstraction $a$ canonically represents a core concept or action, capturing what the memory is fundamentally about and serving as the stable organizing unit of memory. It allows related information, such as recurring events or evolving entity states, to be consolidated under a single persistent entry rather than fractured across redundant records.

The construction of the memory entries along with the primary abstraction follows a two-stage process: extraction and consolidation. Given a new input segment $s$, we first induce a set of candidate memory entries, each consisting of a proposed abstraction and its concrete content:
\begin{equation}
\mathcal{F}_a(s) = \{ m_{i} \}_{i=1}^{N}, \qquad m_i = (a_i, v_i),
\end{equation}
where $a_i$ represents the primary abstraction and $v_i$ denotes the corresponding memory value, which stores the concrete details. This step proposes \emph{potential} new memories prior to verification against the existing store.

In the consolidation phase, we integrate these candidates into $\mathcal{M}$. For a new candidate memory entry $m_i$, we first retrieve top-$k$ existing entries most similar to the induced abstraction $a_i$:
\begin{equation}
\mathcal{R}(a_i) = \operatorname{TopK}_{m \in \mathcal{M}}\bigl(\mathrm{sim}(a_i, a_m);\, k\bigr),
\end{equation}
where $\mathrm{sim}(\cdot,\cdot)$ denotes cosine similarity between the primary abstraction embeddings. We refine this set by filtering out candidates below a similarity threshold $\gamma$:

\begin{equation}
\mathcal{U}(a_i) = \{ m \in \mathcal{R}(a_i) \mid \mathrm{sim}(a_i, a_m) \ge \gamma \}.
\end{equation}

Next, an LLM-based selection function $\mathcal{J}$ determines if the new candidate $(a_i, v_i)$ refers to the same underlying concept as any retrieved entry in $\mathcal{U}(a_i)$:
\begin{equation}
m^\star(a_i) = \mathcal{J}\bigl(a_i, \mathcal{U}(a_i)\bigr).
\end{equation}
Here $\mathcal{J}(\cdot)$ returns the target memory entry $m^\star(a_i)$ if a match is found, or $\varnothing$ if the abstraction $a_i$ is a novel concept.

The final memory construction operation follows a create-or-update rule:
\begin{equation}
m_i =
\begin{cases}
\mathrm{Update}\!\left(m^\star(a_i),\, a_i, v_i\right), & m^\star(a_i) \neq \varnothing,\\[3pt]
\mathrm{Create}\!\left(a_i, v_i\right), & m^\star(a_i) = \varnothing.
\end{cases}
\end{equation}

When a match $m^\star(a_i)$ is found, the $\mathrm{Update}(\cdot)$ operation merges the new content $v_i$ into the existing memory $m^\star(a_i)$, potentially also refining its abstraction to reflect the aggregated information, yielding an updated abstraction $a'_i$. Otherwise, $\mathrm{Create}(\cdot)$ initializes a new memory entry. This policy ensures that each memory entry remains anchored to a single primary abstraction, while enabling new information semantically aligned with existing content to be incrementally incorporated. As a result, the system enriches existing concepts with new details where possible, establishing new abstractions only when necessary.

\subsection{Cue Anchors}

While primary abstractions provide stable and compact organization of memory, they are intentionally coarse and do not capture all task-relevant details needed for flexible retrieval. To address this limitation, \textsc{Memora} introduces \emph{cue anchors}, which serve as lightweight, fine-grained semantic hooks that complement primary abstractions by exposing additional retrieval paths into memory.

Given a memory entry $m_i = (a_i, v_i)$ constructed in the previous step, cue anchors are generated to capture additional salient signals not explicitly represented by the primary abstraction. Formally, cue anchor generation is defined as
\begin{equation}
\mathcal{F}_c(a_i, v_i) = \{ c_{ij} \}_{j=1}^{|\mathcal{C}_i|}, \qquad c_{ij} \in \mathcal{C}_i,
\end{equation}
where the resulting set $\mathcal{C}_i$ contains the cue anchors associated with memory entry $m_i$. Each cue anchor represents a salient aspect, attribute, or contextual perspective of the memory content, formatted as a composite of a main entity/topic and a key aspect. Unlike primary abstractions, which define the canonical identity of a memory entry, cue anchors are non-exclusive and form a many-to-many mapping: a single memory entry may be associated with multiple cue anchors, and the same cue anchor may appear across multiple memory entries.

When new cue anchors are generated, we perform an existence check against the memory store. If an anchor already exists, we simply link memory entry to the existing instance; otherwise, a new anchor is instantiated. Conversely, when memory entries are removed or merged, the corresponding cue–memory links are also updated. Any cue anchor that loses all associations is automatically pruned, ensuring the cue anchors remain compact and non-redundant.

\section{Policy-Guided Memory Retrieval}

Standard retrieval methods, such as semantic search \citep{karpukhin2020densepassageretrievalopendomain}, often fail to capture the multi-hop dependencies required for complex reasoning. To address this, we formulate memory retrieval in \textsc{Memora} as a Markov Decision Process (MDP) \citep{puterman2014markov}. Unlike static semantic search, a policy-guided retriever actively navigates the memory structure to construct a compact yet informative memory set $\mathcal{M}_q$ under a finite budget.

\subsection{Memory Retrieval Policy Formulation}

To operationalize the retrieval process, we define a step-by-step procedure where an agent iteratively observes the current state and selects actions to refine its memory set. The overall process is outlined in Algorithm~\ref{alg:policy_retrieval}.

\begin{algorithm}[t]
\caption{Policy-Guided Sequential Retrieval}
\label{alg:policy_retrieval}
\begin{algorithmic}[1]
\REQUIRE Query $q$, memory system $\mathcal{S}$, policy $\pi_\theta$, budget $B$, max steps $T$
\STATE Initialize $q_0 \leftarrow q$, $\mathcal{M}_0 \leftarrow \emptyset$, $b_0 \leftarrow B$
\STATE Initialize frontier $\mathcal{F}_0 \leftarrow \mathrm{InitFrontier}(q_0, \mathcal{S})$
\FOR{$t = 0,1,\ldots,T-1$}
    \STATE $s_t \leftarrow (q_t, \mathcal{W}_t, \mathcal{F}_t, b_t)$
    \STATE Select $a_t \sim \pi_\theta(\cdot \mid s_t)$
    \IF{$a_t = \textsc{Stop}$ \OR $b_t \le 0$}
        \STATE \textbf{break}
    \ENDIF
    \STATE $(\Delta \mathcal{W}_t, \Delta \mathcal{F}_t, q_{t+1}) \leftarrow \mathrm{Apply}(a_t, s_t, \mathcal{S})$
    \STATE $\mathcal{W}_{t+1} \leftarrow \mathcal{W}_t \cup \Delta \mathcal{W}_t$
    \STATE $\mathcal{F}_{t+1} \leftarrow \mathrm{UpdateFrontier}(\mathcal{F}_t, \Delta \mathcal{F}_t)$
    \STATE $b_{t+1} \leftarrow b_t - \mathrm{Cost}(a_t)$
\ENDFOR
\STATE $\mathcal{M}_{q} \leftarrow \mathcal{W}_t$
\STATE \textbf{return} Retrieved memories $\mathcal{M}_q$
\end{algorithmic}
\end{algorithm}

Given a query $q$ and a retrieval budget $B$, the system state at step $t$ is defined as
\begin{equation}
s_t = (q_t, \mathcal{W}_t, \mathcal{F}_t, b_t).
\end{equation}
Here, $q_t$ is the current query representation, which can be refined over time; $\mathcal{W}_t$ represents the working set of memory entries retrieved so far; $\mathcal{F}_t$ is the frontier, representing a set of candidate memories explicitly linked to items in $\mathcal{W}_t$ but not yet retrieved, allowing the agent to observe what is reachable; and $b_t$ is the remaining retrieval budget.

At each step, the policy $\pi_\theta(a_t \mid s_t)$ selects an action $a_t$ from three atomic retrieval-control operations: \textsc{Refine}, \textsc{Expand}, and \textsc{Stop}. \textsc{Refine} regenerates or reformulates the query when the policy determines that the current query is insufficient or misaligned. This allows the agent to pivot its search strategy to target alternative information relevant to the final answer. \textsc{Expand} expands the working set by selecting relevant memories from the frontier $\mathcal{F}_t$. This action directly grows the working set with new evidence. \textsc{Stop} terminates the retrieval process when sufficient information has been gathered.

% \textit{Retrieval actions} in $\mathcal{A}_{\text{ret}}$ directly query the memory store to retrieve entries relevant to the current query. Specifically, \textsc{QueryA} retrieves entries by matching against primary abstractions, enabling coarse-grained, concept-level retrieval, while \textsc{QueryC} retrieves entries via cue anchors, providing fine-grained access based on contextual or attribute-level signals. \textit{Traversal actions} in $\mathcal{A}_{\text{trv}}$ navigate the internal structure of the memory system by moving along semantic links in the retrieval frontier. \textsc{Traverse}\(_{A\!\rightarrow\!C}\) explores cue anchors associated with a selected primary abstraction, exposing additional retrieval perspectives. \textsc{Traverse}\(_{C\!\rightarrow\!A}\) moves from a cue anchor to linked primary abstractions, enabling abstraction-level generalization from specific cues. \textsc{Traverse}\(_{C\!\rightarrow\!C}\) traverses between cue anchors based on semantic relations, allowing lateral exploration among related contextual signals. \textit{Control actions} in $\mathcal{A}_{\text{ctl}}$ regulate the retrieval trajectory. The \textsc{RewriteQ} action updates the query representation based on previously retrieved content to refine subsequent retrieval decisions, while the \textsc{Stop} action terminates the retrieval process when sufficient information has been gathered or the budget is exhausted.

Executing an action $a_t$ triggers the transition:
\begin{equation}
\mathrm{Apply}(a_t, s_t, \mathcal{S}) \rightarrow s_{t+1}.
\end{equation}
The working set accumulates new retrieved results, and the frontier is updated to include the neighbors of these newly retrieved items:
\[
\begin{aligned}
\mathcal{W}_{t+1} &= \mathcal{W}_t \cup \Delta \mathcal{W}_t, \\
\mathcal{F}_{t+1} &= \mathrm{UpdateFrontier}(\mathcal{F}_t, \Delta \mathcal{F}_t).
\end{aligned}
\]
Simultaneously, the remaining budget is reduced according to the cost of the selected action:
\begin{equation}
b_{t+1} = b_t - \mathrm{Cost}(a_t).
\end{equation}

The retrieval process terminates when either the \textsc{Stop} action is selected or the budget is exhausted. The accumulated working set $\mathcal{W}_t$ is returned as the final retrieved memory context $\mathcal{M}_q$.
%This formulation satisfies the Markov property, as each state transition depends only on the current state and action, enabling policy-guided exploration of the memory space that adapts dynamically to previously retrieved content and the evolving frontier.

\subsection{Group-Relative Policy Updates}

The policy $\pi_\theta$ can be implemented in various ways, ranging from a prompt-guided LLM (zero-shot) to a fully trained retrieval model. While prompt-guided policies based on off-the-shelf models can be directly applied for memory retrieval, they often fail to optimally balance retrieval cost against information gain. In this paper, we also explore optimizing the retrieval policy via group relative policy updates \cite{shao2024deepseekmathpushinglimitsmathematical}.

We treat retrieval as a preference learning problem. Given a query $q$, we sample a group of $G$ retrieval trajectories
\begin{equation}
\mathcal{T}_q \triangleq \{\tau^{(i)}\}_{i=1}^{G}, \qquad
\tau^{(i)} = \{(s_t^{(i)}, a_t^{(i)})\}_{t=0}^{T_i}.
\end{equation}
using the current policy $\pi_\theta$, optionally mixed with a reference policy for exploration. 
% Each trajectory yields a retrieved memory set $\mathcal{W}^{(i)}$.

A trajectory-level judge assigns a scalar score $J(\tau^{(i)})$ to each trajectory based on three criteria: 
(i) correctness of the final answer, 
(ii) information redundancy among retrieved memories, and 
(iii) retrieval cost.
%Importantly, no step-wise supervision is required.

To reduce variance and dependence on absolute scalar rewards, we compute group-relative advantages within each query group:
\begin{equation}
\tilde{A}^{(i)}
=
J(\tau^{(i)})
-
\frac{1}{G}
\sum_{i'=1}^{G}
J(\tau^{(i')}).
\end{equation}
This normalization yields zero-mean advantages within each group, improving robustness to score scaling and judge bias while encouraging relative improvement among trajectories generated for the same query.

The retrieval policy is updated to increase the likelihood of actions from trajectories with positive relative advantage:
\begin{equation}
\mathcal{L}_{\mathrm{GR}}(\theta)
=
-
\sum_{i=1}^{G}
\tilde{A}^{(i)}
\sum_t
\log \pi_\theta\!\left(a_t^{(i)} \mid s_t^{(i)}\right).
\end{equation}

To stabilize training and prevent policy drift, we optionally regularize the update with a KL constraint relative to a reference policy $\pi_{\mathrm{ref}}$:
\begin{equation}
\mathcal{L}(\theta)
=
\mathcal{L}_{\mathrm{GR}}(\theta)
+
\beta
\sum_t
\mathrm{KL}\!\left(
\pi_\theta(\cdot \mid s_t)
\;\|\;
\pi_{\mathrm{ref}}(\cdot \mid s_t)
\right).
\end{equation}

% This group-relative formulation enables effective preference-based optimization of retrieval policies under sparse supervision and aligns naturally with the MDP-based sequential retrieval framework.

This formulation enables preference-based optimization under sparse supervision and aligns naturally with the MDP-based sequential retrieval framework.

\section{Experiments}
\label{sec:experiments}

We conduct extensive experiments to evaluate the effectiveness of \textsc{Memora} on long-context reasoning tasks.

\begin{table*}[t!]
\centering

\caption{Performance comparison on the LoCoMo dataset. Results for Zep, LangMem and Nemori are reported from \citet{nan2025nemori}. \textsc{Memora} (S) and \textsc{Memora} (P) denote the results obtained using the semantic retriever and policy retriever, respectively.}

\resizebox{\textwidth}{!}{
\begin{tabular}{l|ccc ccc ccc ccc|ccc}
\toprule
 & \multicolumn{3}{c}{Multi-hop} 
 & \multicolumn{3}{c}{Temporal} 
 & \multicolumn{3}{c}{Open-domain} 
 & \multicolumn{3}{c|}{Single-hop} 
 & \multicolumn{3}{c}{Overall} \\
\cmidrule(lr){2-4}
\cmidrule(lr){5-7}
\cmidrule(lr){8-10}
\cmidrule(lr){11-13}
\cmidrule(lr){14-16}
\multicolumn{1}{c|}{Method}

& BLEU & F1 & LLM 
& BLEU & F1 & LLM 
& BLEU & F1 & LLM 
& BLEU & F1 & LLM 
& BLEU & F1 & LLM \\
\midrule
Full Context
& 0.356 & 0.459 & 0.766
& 0.506 & 0.572 & 0.819
& 0.204 & 0.250 & 0.500
& 0.557 & 0.634 & 0.885
& 0.487 & 0.565 & 0.825 \\

\midrule

RAG
& 0.222 & 0.324 & 0.557
& 0.428 & 0.486 & 0.548
& 0.224 & 0.277 & 0.458
& 0.448 & 0.507 & 0.710
& 0.389 & 0.455 & 0.633 \\

HippoRAG
& 0.039 & 0.060 & 0.390
& 0.013 & 0.022 & 0.224
& 0.033 & 0.068 & 0.510
& 0.037 & 0.101 & 0.587
& 0.032 & 0.075 & 0.471 \\

Zep* 
& 0.204 & 0.305 & 0.537
& 0.200 & 0.239 & 0.602
& 0.193 & 0.242 & 0.438
& 0.400 & 0.455 & 0.669
& 0.309 & 0.369 & 0.616 \\

Mem0
& 0.236 & 0.326 & 0.624
& 0.420 & 0.489 & 0.660
& 0.153 & 0.206 & 0.500
& 0.376 & 0.433 & 0.677
& 0.346 & 0.411 & 0.653 \\

LangMem* 
& 0.325 & 0.415 & 0.710
& 0.409 & 0.485 & 0.508
& 0.264 & 0.328 & 0.590
& 0.436 & 0.510 & 0.845
& 0.400 & 0.476 & 0.734 \\

Nemori*
& 0.319 & 0.417 & 0.751
& 0.502 & 0.577 & 0.776
& 0.193 & 0.258 & 0.510
& 0.515 & 0.588 & 0.849
& 0.456 & 0.534 & 0.794 \\

% EverMemOS
% & -- & -- & 0.762
% & -- & -- & 0.834
% & -- & -- & 0.479
% & -- & -- & 0.899
% & -- & -- & 0.834 \\
\midrule
\textsc{Memora} (S)
& 0.321 & 0.417 & 0.784
& 0.502 & 0.624 & 0.851
& 0.251 & 0.318 & \textbf{0.594}
& 0.522 & 0.597 & 0.900
& 0.464 & 0.552 & 0.849 \\

\textsc{Memora} (P)
& 0.337 & 0.428 & \textbf{0.787}
& 0.500 & 0.623 & \textbf{0.866}
& 0.246 & 0.308 & \textbf{0.594}
& 0.521 & 0.597 & \textbf{0.918}
& 0.466 & 0.553 & \textbf{0.863} \\
% \midrule
% \textsc{Memora} (P + A)
% & - & - & \textbf{0.926}
% & - & - & \textbf{0.879}
% & - & - & \textbf{0.792}
% & - & - & \textbf{0.962}
% & - & - & \textbf{0.927} \\
\bottomrule
\end{tabular}
}

\label{tab:locomo}

\end{table*}

\begin{table}[t]
\centering
\caption{Performance comparison on LongMemEval.}

\resizebox{\columnwidth}{!}{
\begin{tabular}{l|cccc}
\toprule
Question Type & Full Context & Nemori & \textsc{Memora}(S) & \textsc{Memora} (P) \\
 \midrule
 Context length & 115k & 3.7-4.8k & 2.1k & 2.9k\\
 \midrule
single-sn-preference & 16.7\% & 86.7\% & 76.7\% & \textbf{83.3\%} \\
single-sn-assistant  & \textbf{98.2\%} & 92.9\% & 76.8\% & 78.6\% \\
temporal-reasoning       & 60.2\% & 72.2\% & 84.2\% & \textbf{89.5\%} \\
multi-session            & 51.1\% & 55.6\% & 73.7\% & \textbf{78.2\%} \\
knowledge-update         & 76.9\% & 79.5\% & 96.2\% & \textbf{97.4\%} \\
single-sn-user      & 85.7\% & 90.0\% & 97.1\% & \textbf{98.6\%} \\ 
\midrule
\textbf{Average}         & 65.6\% & 74.6\% & 83.8\% & \textbf{87.4\%} \\
\bottomrule
\end{tabular}
}
\label{tab:longmemeval}
\end{table}

\subsection{Experimental Setup}

\textbf{Datasets.} We evaluate our method on two long-context and multi-session reasoning benchmarks. \textbf{LoCoMo} \citep{maharana2024evaluating} comprises extensive multi-turn dialogues averaging 600 turns ($\sim$20k tokens). It challenges models with diverse question-answer pairs spanning single-hop, multi-hop, temporal, and open-domain tasks, requiring the synthesis of information across long conversational histories. \textbf{LongMemEval} \citep{wu2024longmemeval} is a comprehensive benchmark for evaluating long-term memory robustness. We use the \texttt{LongMemEval\_S} split (115k context length), which contains 500 questions derived from user--assistant interactions to test reasoning over extreme context windows.

\textbf{Baselines.} We compare \textsc{Memora} against a diverse set of baselines representing current state-of-the-art approaches:
(1) \textit{Full Context} that feed the entire context history into the prompt.
(2) \textit{RAG} that chunks context history and retrieves top-$k$ fragments ($chunksize=500$ and $k=3$).
(3) \textit{Memory Systems} including \textbf{HippoRAG}, \textbf{Zep}, \textbf{Mem0}, \textbf{LangMem}, and \textbf{Nemori}, 
    %and \textbf{EverMemOS} \citep{hu2026evermemos}
    which utilize various strategies for memory management.

\textbf{Evaluation Metrics.} We report the \textbf{LLM-as-a-Judge} score as our \textit{primary metric}, as it best captures the semantic validity of the generated answers. To ensure fair comparison, we adopt the same evaluation templates from prior work to assess the correctness of the responses. Full evaluation setup is detailed in Appendix~\ref{app:eval_setup}. We report \textbf{BLEU} and \textbf{F1} scores as complementary metrics on the LoCoMo dataset to measure the verbatim overlap between answers and the ground truth.

\textbf{Retrieval Configurations.} We evaluate \textsc{Memora} using three retrieval mechanisms: (1) \textit{Semantic Retriever} (S), retrieval based on semantic similarity; (2) \textit{Policy Retriever} (P), retrieval guided by a prompt-based LLM agent; (3) \textit{GRPO Retriever}, retrieval guided by a policy trained via GRPO. To accommodate the training requirements of the GRPO variant, we employ two evaluation setups. For our main results and ablation studies, we evaluate the \textit{Semantic} and \textit{Policy} retrievers on the \textbf{full} LoCoMo and LongMemEval datasets. For the GRPO experiments, we partition the LoCoMo dataset into train and test splits. We report GRPO metrics exclusively on the test partition to quantify the specific gains from policy optimization.

\textbf{Implementation Details.} All experiments utilize \texttt{GPT-4.1-mini} as the LLM backbone for memory curation, answer generation, as well as prompt-based policy retrieval. To ensure reproducibility, we fix the generation seed to 42 across all runs. Prompts used for memory extraction are provided in Appendix \ref{app:prompts}.

\subsection{Results and Analysis}

\begin{table}[t!]
\centering
\caption{Component build-up ablation on LoCoMo}
\resizebox{\columnwidth}{!}{
\begin{tabular}{l c c c c c}
\toprule
Method & Multi & Temp & Open & Single & Overall \\
\midrule
\textsc{Memora} w/o abstraction ($=$ Mem0)              & 0.624 & 0.660 & 0.500 & 0.677 & 0.653 \\
\textsc{Memora} (primary abstraction, no update)        & 0.727 & 0.801 & 0.542 & 0.844 & 0.795 \\
\textsc{Memora} (primary abstraction, with update)      & 0.780 & 0.813 & 0.510 & 0.836 & 0.801 \\
\textsc{Memora} (semantic retriever)                    & 0.784 & 0.851 & \textbf{0.594} & 0.900 & 0.849 \\
\textsc{Memora} (policy retriever)                      & \textbf{0.787} & \textbf{0.866} & \textbf{0.594} & \textbf{0.918} & \textbf{0.863} \\
\bottomrule
\end{tabular}
}
\label{tab:abstraction_buildup}
\end{table}

\begin{table*}[t!]
\centering
\caption{Ablation studies on retrieval policy and memory granularity the LoCoMo dataset.}
\resizebox{\textwidth}{!}{
\begin{tabular}{l|c|ccc ccc ccc ccc|ccc}
\toprule
 &  & \multicolumn{3}{c}{Multi-hop} 
 & \multicolumn{3}{c}{Temporal} 
 & \multicolumn{3}{c}{Open-domain} 
 & \multicolumn{3}{c|}{Single-hop} 
 & \multicolumn{3}{c}{Overall} \\
\cmidrule(lr){3-5}
\cmidrule(lr){6-8}
\cmidrule(lr){9-11}
\cmidrule(lr){12-14}
\cmidrule(lr){15-17}
\multicolumn{1}{c|}{Method}
& \multicolumn{1}{c|}{Avg. Tokens}
& BLEU & F1 & LLM 
& BLEU & F1 & LLM 
& BLEU & F1 & LLM 
& BLEU & F1 & LLM 
& BLEU & F1 & LLM \\

\midrule
\multicolumn{17}{l}{\textbf{Policy Retriever}} \\
\midrule

Episodic (Segment) + Factual 
& 8499
& 0.337 & 0.428 & 0.787 & 0.500 & 0.623 & \textbf{0.866} & 0.246 & 0.308 & 0.594 & 0.521 & 0.597 & \textbf{0.918} & 0.466 & 0.553 & \textbf{0.863} \\

Episodic (Segment) only
& 6624
& 0.350 & 0.451 & 0.780 & 0.517 & 0.610 & 0.847 & 0.260 & 0.328 & \textbf{0.625} & 0.544 & 0.619 & 0.903 & 0.485 & 0.568 & 0.851 \\

Episodic (Segment) + Factual w/o cue 
& 8425
& 0.329 & 0.416 & 0.773 & 0.512 & 0.631 & 0.857 & 0.243 & 0.299 & 0.594 & 0.518 & 0.596 & 0.905 & 0.465 & 0.552 & 0.851 \\

Episodic (Extracted) + Factual 
& 4467
& 0.328 & 0.417 & 0.762 & 0.521 & 0.646 & 0.860 & 0.245 & 0.303 & 0.615 & 0.475 & 0.543 & 0.880 & 0.443 & 0.526 & 0.838 \\

Factual only 
& 1853
& 0.309 & 0.398 & \textbf{0.801} & 0.522 & 0.646 & 0.851 & 0.225 & 0.277 & 0.542 & 0.484 & 0.551 & 0.870 & 0.444 & 0.526 & 0.833 \\

\midrule
\multicolumn{17}{l}{\textbf{Semantic Retriever}} \\
\midrule

Episodic (Segment) + Factual 
& 7683
& 0.321 & 0.417 & 0.784 & 0.502 & 0.624 & 0.851 & 0.251 & 0.318 & 0.594 & 0.522 & 0.597 & \textbf{0.900} & 0.464 & 0.552 & 0.849 \\

Episodic (Segment) only
& 6042
& 0.349 & 0.450 & 0.773 & 0.506 & 0.599 & 0.832 & 0.260 & 0.325 & \textbf{0.615} & 0.539 & 0.614 & 0.899 & 0.480 & 0.563 & 0.844 \\

Episodic (Segment) + Factual w/o cue 
& 7628
& 0.338 & 0.434 & 0.780 & 0.511 & 0.635 & 0.854 & 0.253 & 0.316 & 0.604 & 0.516 & 0.589 & \textbf{0.900} & 0.466 & 0.553 & \textbf{0.850} \\

Episodic (Extracted) + Factual 
& 3958
& 0.315 & 0.406 & 0.755 & 0.523 & 0.646 & \textbf{0.857} & 0.224 & 0.282 & 0.573 & 0.477 & 0.542 & 0.875 & 0.441 & 0.522 & 0.831 \\

Factual only 
& 1647
& 0.309 & 0.402 & \textbf{0.791} & 0.526 & 0.647 & 0.847 & 0.210 & 0.265 & 0.531 & 0.481 & 0.546 & 0.857 & 0.442 & 0.523 & 0.823 \\

\bottomrule
\end{tabular}
}
\label{tab:ablation}
\end{table*}

\subsubsection{Performance Analysis}
Table~\ref{tab:locomo} presents the results on the LoCoMo dataset. Our best-performing configuration, \textsc{Memora} with the \textit{Policy Retriever}, achieves a score of \textbf{0.863}, followed by the \textit{Semantic Retriever} variant at 0.849. \textsc{Memora} demonstrates superior performance across all four task categories, establishing a new state-of-the-art.

Notably, \textsc{Memora} surpasses the \textit{Full Context} baseline (0.825). We attribute this result to Memora's ability to reduce ``context noise". By filtering out irrelevant dialogue turns and presenting a crystallized memory structure, \textsc{Memora} prevents the dilution of the model's attention mechanism, effectively proving that \textit{curated} context leads to sharper reasoning than \textit{complete} context. \textsc{Memora} also significantly outperforms strong baselines, including RAG (0.633), as well as other competitive memory systems such as Mem0 (0.653) and Nemori (0.794). This performance gap validates the utility of our harmonic structure.
%As detailed in the case study (Appendix~\ref{app:case_study}), this success is driven by the synergy between our components: while the primary abstraction and cue anchors enable the model to \textit{pinpoint} targets with high precision, the underlying index-value representation ensures the optimal balance between specificity and abstraction. 
The Policy Retriever further amplifies these gains by leveraging cue anchors to actively navigate the memory graph, ensuring that contextually linked information is retrieved even when it is not semantically adjacent.

Table \ref{tab:longmemeval} presents the performance on the LongMemEval dataset, where our method consistently outperforms strong baselines, achieving an accuracy of 87.4\%.

\subsubsection{Ablation Studies}
To understand the contribution of each component in \textsc{Memora}, we conduct two complementary ablations. Table~\ref{tab:abstraction_buildup} builds the system up from a no-abstraction baseline to the full model, quantifying what each architectural component contributes; Table~\ref{tab:ablation} then performs a leave-one-out study over the full system, varying retrieval policy, memory types, and granularity.

Removing the primary-abstraction layer essentially reduces \textsc{Memora} to Mem0: both extract factual memories from data, but Mem0 embeds and retrieves the extracted text fragments directly, whereas \textsc{Memora} decouples \emph{what} is stored (the memory content) from \emph{how} it is retrieved by indexing through a structural abstraction layer (primary abstractions and cue anchors), allowing more expressiveness in the stored content and avoiding fragmentation. Even with cue anchors and policy retrieval disabled, adding the primary-abstraction layer alone lifts the overall LLM-as-a-Judge score from 0.653 to 0.795. Adding update, cue anchors and policy retrieval then provides further gains. This decoupling also produces far fewer memory entries (344 for \textsc{Memora} vs.\ 651 for Mem0 per conversation on average), which in turn improves retrieval efficiency. Direct embedding of memory content tends toward two extremes: Mem0 produces fragmented, isolated entries that lose inter-memory coherence, while RAG encodes rich content into a single fuzzy vector, preserving expressiveness at the cost of retrieval precision. \textsc{Memora}'s abstraction layer sidesteps both extremes by storing expressive, unfragmented content while indexing through lightweight abstractions. Appendix~\ref{app:case_study} provides three case studies illustrating these dynamics.

Comparing the two major retriever backbones in Table~\ref{tab:ablation}, the policy retriever consistently outperforms the semantic retriever. Crucially, this advantage disappears when cue anchors are removed, rendering the policy retriever comparable to the semantic approach. This highlights that the improvement is not merely a consequence of increased complexity in the policy network, but rather stems from its capacity to leverage cue anchors for traversing the memory graph. By following these anchors, the system can navigate to relevant non-local contexts that a semantic search would miss.

We also examine the impact of context granularity. We observe a clear performance hierarchy correlated with the richness of the episodic context: the variant using raw segments as episodic memory (\textit{Episodic (Segment) + Factual}) achieves the highest score (0.863), outperforming the extracted episodic memory (\textit{Episodic (Extracted) + Factual}, 0.838) and the \textit{Factual Only} variant (0.833). This trend confirms that while discrete facts provide a solid baseline, the ``connective tissue" found in episodic memory is essential for grounding. Furthermore, factual and episodic memories are not redundant but complementary. Adding factual memory to the episodic-only baseline consistently improves overall performance, indicating that \textsc{Memora} succeeds by combining the structural clarity of factual details with the richer context of the episodes.

Finally, we note the trade-off between performance and memory size. While the full \textit{Episodic (Segment) + Factual} variant yields the best results, greater context richness inevitably leads to a larger memory footprint. However, the \textit{Factual-only} configuration remains a strong ``lightweight" alternative, achieving a respectable score of 0.833 while significantly reducing the context load. This highlights \textsc{Memora}'s flexibility for either maximum contextual fidelity or efficiency, depending on resource constraints.

\subsubsection{Inference Latency Analysis}

\begin{table}[t!]
\centering
\caption{Latency on the LoCoMo dataset. \textit{End-to-end Latency} refers to the full inference workflow for each query, while \textit{Search Latency} measures the memory retrieval steps.}
\resizebox{\columnwidth}{!}{
\begin{tabular}{l c c c c c c c}
\toprule
& \multicolumn{3}{c}{End-to-end Latency (s)} 
& \multicolumn{3}{c}{Search Latency (s)} 
& \\ 
\cmidrule(lr){2-4} \cmidrule(lr){5-7} 

Method & Mean & P50 & P95 & Mean & P50 & P95 & Avg Steps \\
\midrule
\multicolumn{8}{l}{\textbf{Policy Retriever}} \\
\midrule

Episodic (S) + Factual 
& 5.697 & 5.004 & 10.974 & 4.609 & 3.857 & 9.581 & 3.45 \\
Episodic (E) + Factual 
& 5.438 & 4.703 & 10.593 & 4.497 & 3.719 & 9.437 & 3.39 \\
Factual only 
& 4.653 & 3.940 & 9.388 & 3.969 & 3.279 & 8.495 & 3.36 \\
\midrule

\multicolumn{8}{l}{\textbf{Semantic Retriever}} \\
\midrule

Episodic (S) + Factual 
& 1.062 & 1.016 & 1.487 & 0.235 & 0.221 & 0.256 & 1 \\
Episodic (E) + Factual 
& 0.958 & 0.908 & 1.336 & 0.232 & 0.221 & 0.260 & 1 \\
Factual only 
& 0.733 & 0.676 & 1.006 & 0.220 & 0.200 & 0.245 & 1 \\
\bottomrule
\end{tabular}
}
\label{tab:latency_results}
\end{table}

Table~\ref{tab:latency_results} details the latency metrics. For latency evaluation, we report the mean, P50 and P95 wall-clock latencies. These metrics capture both end-to-end response generation and retrieval operations across the LoCoMo dataset, accounting for real-world API overhead. We report these metrics across three memory configurations: \textit{Episodic (\textbf{S}egment) + Factual}, \textit{Episodic (\textbf{E}xtracted) + Factual}, and \textit{Factual Only}, as they represent different memory sizes. The policy retriever incurs higher latency compared to the semantic retriever, primarily due to the sequential nature of the search process. On average, the policy retriever requires over three steps per query. Since each step involves a distinct LLM call to determine the next action, the search latency naturally scales with the number of iterations.

\subsubsection{Memory Construction Analysis}
\begin{table}[t!]
\centering
\caption{Average memory construction time per conversation on LoCoMo.}
\resizebox{0.7\columnwidth}{!}{
\begin{tabular}{l c c}
\toprule
System & Time (s) & Performance \\
\midrule
Mem0                        & 1350.9         & 0.653 \\
\textsc{Memora}             & 1322.0         & \textbf{0.863} \\
\textsc{Memora} (offset) & \textbf{739.9} & 0.860 \\
\bottomrule
\end{tabular}
}
\label{tab:construction_overhead}
\end{table}

\textbf{Construction overhead.} Table~\ref{tab:construction_overhead} reports the average memory construction time per conversation on LoCoMo together with the corresponding overall performance score. Before optimization, \textsc{Memora}'s construction cost (1322.0s) is comparable to Mem0 (1350.9s), but with substantially higher answer quality (0.863 vs. 0.653), showing that the performance gains thus arise from a better memory representation rather than additional compute. We further apply an optimization that predicts index offsets into the source material rather than generating full memory values, which reduces construction time to 739.9s, a 45\% speedup with minimal loss in quality.

\begin{table}[t!]
\centering
\caption{Using smaller memory-construction model on LoCoMo.}
\resizebox{\columnwidth}{!}{
\begin{tabular}{l l c c c c c}
\toprule
Model & Retrieval & Multi & Temp & Open & Single & Overall \\
\midrule
\texttt{gpt-5.4-nano} & Semantic & 0.713 & 0.620 & 0.479 & 0.867 & 0.763 \\
\texttt{gpt-4.1-mini} & Semantic & 0.784 & 0.851 & 0.594 & 0.900 & 0.849 \\
\texttt{gpt-5.4-nano} & Policy   & 0.773 & 0.879 & 0.625 & 0.893 & 0.851 \\
\texttt{gpt-4.1-mini} & Policy   & 0.787 & 0.866 & 0.594 & 0.918 & \textbf{0.863} \\
\bottomrule
\end{tabular}
}
\label{tab:smaller_llm}
\end{table}

\textbf{Robustness to smaller construction LLMs.} A natural concern is whether \textsc{Memora}'s gains depend on a powerful LLM at construction time. We replace the default \texttt{gpt-4.1-mini} with the less capable \texttt{gpt-5.4-nano} during memory construction; results are shown in Table~\ref{tab:smaller_llm}. First, while a weaker construction model does produce slightly lower-quality memories, even the worst-case (nano + semantic, 0.763) \textsc{Memora} still substantially outperforms baselines that use \texttt{gpt-4.1-mini} for memory construction, including Mem0 (0.653) and RAG (0.633). This indicates that \textsc{Memora}'s advantage stems from its structural design rather than reliance on a powerful construction model. Second, policy-guided retrieval effectively recovers the construction-quality gap: nano + policy (0.851) drops only 1.4\% from mini + policy (0.863) and nearly matches mini + semantic (0.849), showing that one can build memories with a much cheaper model and recover quality through policy-guided retrieval.

\textbf{Memory Update Analysis.} \textsc{Memora}'s update mechanism merges a newly extracted memory into an existing entry when their primary abstractions exceed a semantic-similarity threshold. We use a default threshold of 0.80, chosen to merge only entries that are clearly redundant while preventing over-aggressive consolidation. We provide an analysis on memory update sensitivity and how memory growth affects update frequency and construction overhead in Appendix \ref{app:update_analysis}. We show that consolidation cost scales linearly with memory bank size and does not become a bottleneck as the memory store grows.

\begin{figure}[t]
\centering
\resizebox{\linewidth}{!}{%
\begin{tikzpicture}
\begin{axis}[
    xbar,
    bar width=12pt,
    width=14cm,
    height=8.5cm,
    xmin=0.4, xmax=0.92,
    xtick={0.4,0.5,0.6,0.7,0.8,0.9},
    xlabel={\large\textbf{LLM-as-a-Judge score}},
    ytick=data,
    yticklabels={
        Multi-hop,
        Temporal,
        Open-domain,
        Single-hop,
        \textbf{Overall}
    },
    tick label style={font=\large},
    enlarge y limits=0.15,
    y dir=reverse,
    legend style={
        at={(0.5,1.02)},
        anchor=south,
        legend columns=2,
        column sep=0.2cm,
        draw=none,
        font=\Large
    },
    grid=major,
    grid style={dashed,gray!30},
    axis line style={black},
    tick style={black},
    nodes near coords,
    nodes near coords align={horizontal},
    every node near coord/.append style={
        anchor=east,
        xshift=-2pt,
        font=\large \bfseries, % Increased font size inside bars
        text=white,
        /pgf/number format/.cd,
            fixed,
            precision=3,
            zerofill
    },
]

\addplot[
    fill={rgb,255:red,213; green,94; blue,0},
    fill opacity=0.85,
    draw=none
] coordinates {
    (0.686,0)
    (0.857,1)
    (0.552,2)
    (0.919,3)
    (0.841,4)
};

\addplot[
    fill={rgb,255:red,0; green,114; blue,178},
    fill opacity=0.85,
    draw=none
] coordinates {
    (0.698,0)
    (0.816,1)
    (0.517,2)
    (0.912,3)
    (0.829,4)
};

\legend{
Qwen 2.5 1.5B (GRPO),
Qwen 2.5 1.5B (Base)
}
\end{axis}
\end{tikzpicture}
}
\caption{Results for GRPO training.}
\label{fig:grpo_results}
\end{figure}

\subsubsection{Policy Training} We further investigate whether the retrieval policy can be explicitly optimized using GRPO. We fine-tune a small backbone model (\texttt{Qwen-2.5-1.5B}) on a 70/30 train/test split of LoCoMo and evaluate on the held-out 30\% test partition. As shown in Figure~\ref{fig:grpo_results}, the GRPO-tuned retriever achieves an overall LLM-as-a-Judge score of 0.841, consistently outperforming the base model baseline (0.829) across temporal, open-domain, and single-hop tasks. These results demonstrate that the retrieval policy is learnable and can be effectively distilled into smaller models, maintaining competitive performance compared to the instruction-tuned counterpart.

% \subsection{Case Study}

\section{Conclusion}
In this work, we introduce \textsc{Memora}, a harmonic memory architecture that balances abstraction and specificity for long-term agent memory. By introducing primary abstractions and cue anchors, \textsc{Memora} enables scalable, context-aware retrieval without fragmenting knowledge or obscuring task-critical detail. A policy-driven retrieval mechanism further allows agents to actively explore relevant memory beyond direct semantic similarity. We show that existing RAG- and KG-based memory systems arise as special cases of our framework. Empirically, Memora achieves state-of-the-art performance on long-horizon memory benchmarks, consistently outperforming strong baselines and full-context inference with both semantic and policy retrieval mechanisms, demonstrating the effectiveness of harmonic memory organization for scalable agent reasoning.

\clearpage
\newpage
\section*{Impact Statement}

This work advances the field of autonomous agents by enabling significantly more consistent and reliable long-term memory systems. By structurally balancing abstraction with specificity, \textsc{Memora} allows agents to retain and utilize context effectively over long horizons, addressing a key bottleneck in current architectures. This improvement in memory management paves the way for the development of a broader range of complex applications, from personalized long-term assistants to collaborative problem-solving system, that require stable and precise context retention. To facilitate reproducibility and further innovation within the community, we release our code at \url{https://github.com/microsoft/Memora}.

\bibliography{memora}

@article{chhikara2025mem0,
  title={Mem0: Building production-ready ai agents with scalable long-term memory},
  author={Chhikara, Prateek and Khant, Dev and Aryan, Saket and Singh, Taranjeet and Yadav, Deshraj},
  journal={arXiv preprint arXiv:2504.19413},
  year={2025},
}

@article{rasmussen2025zep,
  title={Zep: a temporal knowledge graph architecture for agent memory},
  author={Rasmussen, Preston and Paliychuk, Pavlo and Beauvais, Travis and Ryan, Jack and Chalef, Daniel},
  journal={arXiv preprint arXiv:2501.13956},
  year={2025},
}

@misc{nan2025nemori,
      title={Nemori: Self-Organizing Agent Memory Inspired by Cognitive Science}, 
      author={Jiayan Nan and Wenquan Ma and Wenlong Wu and Yize Chen},
      year={2025},
      eprint={2508.03341},
      archivePrefix={arXiv},
      primaryClass={cs.AI},
}

@article{maharana2024evaluating,
  title={Evaluating very long-term conversational memory of llm agents},
  author={Maharana, Adyasha and Lee, Dong-Ho and Tulyakov, Sergey and Bansal, Mohit and Barbieri, Francesco and Fang, Yuwei},
  journal={arXiv preprint arXiv:2402.17753},
  year={2024}
}

@article{wu2024longmemeval,
      title={LongMemEval: Benchmarking Chat Assistants on Long-Term Interactive Memory}, 
      author={Di Wu and Hongwei Wang and Wenhao Yu and Yuwei Zhang and Kai-Wei Chang and Dong Yu},
      year={2024},
      eprint={2410.10813},
      archivePrefix={arXiv},
      primaryClass={cs.CL},
      url={https://arxiv.org/abs/2410.10813}, 
}

@article{packer2023memgpt,
  title={{MemGPT}: Towards LLMs as Operating Systems},
  author={Packer, Charles and Wooders, Sarah and Lin, Kevin and Fang, Vivian and Patil, Shishir G. and Stoica, Ion and Gonzalez, Joseph E.},
  journal={arXiv preprint arXiv:2310.08560},
  year={2023}
}

@misc{lewis2021retrieval,
      title={Retrieval-Augmented Generation for Knowledge-Intensive NLP Tasks}, 
      author={Patrick Lewis and Ethan Perez and Aleksandra Piktus and Fabio Petroni and Vladimir Karpukhin and Naman Goyal and Heinrich Küttler and Mike Lewis and Wen-tau Yih and Tim Rocktäschel and Sebastian Riedel and Douwe Kiela},
      year={2021},
      eprint={2005.11401},
      archivePrefix={arXiv},
      primaryClass={cs.CL},
      url={https://arxiv.org/abs/2005.11401}, 
}

@misc{borgeaud2022improving,
      title={Improving language models by retrieving from trillions of tokens}, 
      author={Sebastian Borgeaud and Arthur Mensch and Jordan Hoffmann and Trevor Cai and Eliza Rutherford and Katie Millican and George van den Driessche and Jean-Baptiste Lespiau and Bogdan Damoc and Aidan Clark and Diego de Las Casas and Aurelia Guy and Jacob Menick and Roman Ring and Tom Hennigan and Saffron Huang and Loren Maggiore and Chris Jones and Albin Cassirer and Andy Brock and Michela Paganini and Geoffrey Irving and Oriol Vinyals and Simon Osindero and Karen Simonyan and Jack W. Rae and Erich Elsen and Laurent Sifre},
      year={2022},
      eprint={2112.04426},
      archivePrefix={arXiv},
      primaryClass={cs.CL},
      url={https://arxiv.org/abs/2112.04426}, 
}

@misc{gao2024retrieval,
      title={Retrieval-Augmented Generation for Large Language Models: A Survey}, 
      author={Yunfan Gao and Yun Xiong and Xinyu Gao and Kangxiang Jia and Jinliu Pan and Yuxi Bi and Yi Dai and Jiawei Sun and Meng Wang and Haofen Wang},
      year={2024},
      eprint={2312.10997},
      archivePrefix={arXiv},
      primaryClass={cs.CL},
      url={https://arxiv.org/abs/2312.10997}, 
}

@misc{xu2025amemagenticmemoryllm,
      title={A-MEM: Agentic Memory for LLM Agents}, 
      author={Wujiang Xu and Zujie Liang and Kai Mei and Hang Gao and Juntao Tan and Yongfeng Zhang},
      year={2025},
      eprint={2502.12110},
      archivePrefix={arXiv},
      primaryClass={cs.CL},
      url={https://arxiv.org/abs/2502.12110}, 
}

@misc{zhong2023memorybank,
      title={MemoryBank: Enhancing Large Language Models with Long-Term Memory}, 
      author={Wanjun Zhong and Lianghong Guo and Qiqi Gao and He Ye and Yanlin Wang},
      year={2023},
      eprint={2305.10250},
      archivePrefix={arXiv},
      primaryClass={cs.CL},
      url={https://arxiv.org/abs/2305.10250}, 
}

@misc{edge2025localglobalgraphrag,
      title={From Local to Global: A Graph RAG Approach to Query-Focused Summarization}, 
      author={Darren Edge and Ha Trinh and Newman Cheng and Joshua Bradley and Alex Chao and Apurva Mody and Steven Truitt and Dasha Metropolitansky and Robert Osazuwa Ness and Jonathan Larson},
      year={2025},
      eprint={2404.16130},
      archivePrefix={arXiv},
      primaryClass={cs.CL},
      url={https://arxiv.org/abs/2404.16130}, 
}

@misc{li2025memos,
      title={MemOS: An Operating System for Memory-Augmented Generation (MAG) in Large Language Models}, 
      author={Zhiyu Li and Shichao Song and Hanyu Wang and Simin Niu and Ding Chen and Jiawei Yang and Chenyang Xi and Huayi Lai and Jihao Zhao and Yezhaohui Wang and Junpeng Ren and Zehao Lin and Jiahao Huo and Tianyi Chen and Kai Chen and Kehang Li and Zhiqiang Yin and Qingchen Yu and Bo Tang and Hongkang Yang and Zhi-Qin John Xu and Feiyu Xiong},
      year={2025},
      eprint={2505.22101},
      archivePrefix={arXiv},
      primaryClass={cs.CL},
      url={https://arxiv.org/abs/2505.22101}, 
}

@misc{wang2025mirix,
      title={MIRIX: Multi-Agent Memory System for LLM-Based Agents}, 
      author={Yu Wang and Xi Chen},
      year={2025},
      eprint={2507.07957},
      archivePrefix={arXiv},
      primaryClass={cs.CL},
      url={https://arxiv.org/abs/2507.07957}, 
}

@misc{yan2026memoryr1,
      title={Memory-R1: Enhancing Large Language Model Agents to Manage and Utilize Memories via Reinforcement Learning}, 
      author={Sikuan Yan and Xiufeng Yang and Zuchao Huang and Ercong Nie and Zifeng Ding and Zonggen Li and Xiaowen Ma and Jinhe Bi and Kristian Kersting and Jeff Z. Pan and Hinrich Schütze and Volker Tresp and Yunpu Ma},
      year={2026},
      eprint={2508.19828},
      archivePrefix={arXiv},
      primaryClass={cs.CL},
      url={https://arxiv.org/abs/2508.19828}, 
}

@article{Wang_2024,
   title={A survey on large language model based autonomous agents},
   volume={18},
   ISSN={2095-2236},
   url={http://dx.doi.org/10.1007/s11704-024-40231-1},
   DOI={10.1007/s11704-024-40231-1},
   number={6},
   journal={Frontiers of Computer Science},
   publisher={Springer Science and Business Media LLC},
   author={Wang, Lei and Ma, Chen and Feng, Xueyang and Zhang, Zeyu and Yang, Hao and Zhang, Jingsen and Chen, Zhiyuan and Tang, Jiakai and Chen, Xu and Lin, Yankai and Zhao, Wayne Xin and Wei, Zhewei and Wen, Jirong},
   year={2024},
   month=mar }

@misc{guo2024largelanguagemodelbased,
      title={Large Language Model based Multi-Agents: A Survey of Progress and Challenges}, 
      author={Taicheng Guo and Xiuying Chen and Yaqi Wang and Ruidi Chang and Shichao Pei and Nitesh V. Chawla and Olaf Wiest and Xiangliang Zhang},
      year={2024},
      eprint={2402.01680},
      archivePrefix={arXiv},
      primaryClass={cs.CL},
      url={https://arxiv.org/abs/2402.01680}, 
}

@misc{wu2023autogen,
      title={AutoGen: Enabling Next-Gen LLM Applications via Multi-Agent Conversation}, 
      author={Qingyun Wu and Gagan Bansal and Jieyu Zhang and Yiran Wu and Beibin Li and Erkang Zhu and Li Jiang and Xiaoyun Zhang and Shaokun Zhang and Jiale Liu and Ahmed Hassan Awadallah and Ryen W White and Doug Burger and Chi Wang},
      year={2023},
      eprint={2308.08155},
      archivePrefix={arXiv},
      primaryClass={cs.AI},
      url={https://arxiv.org/abs/2308.08155}, 
}

@misc{yao2023react,
      title={ReAct: Synergizing Reasoning and Acting in Language Models}, 
      author={Shunyu Yao and Jeffrey Zhao and Dian Yu and Nan Du and Izhak Shafran and Karthik Narasimhan and Yuan Cao},
      year={2023},
      eprint={2210.03629},
      archivePrefix={arXiv},
      primaryClass={cs.CL},
      url={https://arxiv.org/abs/2210.03629}, 
}

@misc{kang2025memoryosaiagent,
      title={Memory OS of AI Agent}, 
      author={Jiazheng Kang and Mingming Ji and Zhe Zhao and Ting Bai},
      year={2025},
      eprint={2506.06326},
      archivePrefix={arXiv},
      primaryClass={cs.AI},
      url={https://arxiv.org/abs/2506.06326}, 
}

@misc{Milam2025context,
    title={Context Engineering: Sessions \& Memory},
    author={Kimberly Milam and Antonio Gulli},
    year={2025}
}

@misc{karpukhin2020densepassageretrievalopendomain,
      title={Dense Passage Retrieval for Open-Domain Question Answering}, 
      author={Vladimir Karpukhin and Barlas Oğuz and Sewon Min and Patrick Lewis and Ledell Wu and Sergey Edunov and Danqi Chen and Wen-tau Yih},
      year={2020},
      eprint={2004.04906},
      archivePrefix={arXiv},
      primaryClass={cs.CL},
      url={https://arxiv.org/abs/2004.04906}, 
}

@book{puterman2014markov,
  title={Markov decision processes: discrete stochastic dynamic programming},
  author={Puterman, Martin L},
  year={2014},
  publisher={John Wiley \& Sons}
}

@misc{shao2024deepseekmathpushinglimitsmathematical,
      title={DeepSeekMath: Pushing the Limits of Mathematical Reasoning in Open Language Models}, 
      author={Zhihong Shao and Peiyi Wang and Qihao Zhu and Runxin Xu and Junxiao Song and Xiao Bi and Haowei Zhang and Mingchuan Zhang and Y. K. Li and Y. Wu and Daya Guo},
      year={2024},
      eprint={2402.03300},
      archivePrefix={arXiv},
      primaryClass={cs.CL},
      url={https://arxiv.org/abs/2402.03300}, 
}

@inproceedings{
gutierrez2024hipporag,
title={Hippo{RAG}: Neurobiologically Inspired Long-Term Memory for Large Language Models},
author={Bernal Jimenez Gutierrez and Yiheng Shu and Yu Gu and Michihiro Yasunaga and Yu Su},
booktitle={The Thirty-eighth Annual Conference on Neural Information Processing Systems},
year={2024},
url={https://openreview.net/forum?id=hkujvAPVsg}
}
\bibliographystyle{icml2026}

%%%%%%%%%%%%%%%%%%%%%%%%%%%%%%%%%%%%%%%%%%%%%%%%%%%%%%%%%%%%%%%%%%%%%%%%%%%%%%%
%%%%%%%%%%%%%%%%%%%%%%%%%%%%%%%%%%%%%%%%%%%%%%%%%%%%%%%%%%%%%%%%%%%%%%%%%%%%%%%
% APPENDIX
%%%%%%%%%%%%%%%%%%%%%%%%%%%%%%%%%%%%%%%%%%%%%%%%%%%%%%%%%%%%%%%%%%%%%%%%%%%%%%%
%%%%%%%%%%%%%%%%%%%%%%%%%%%%%%%%%%%%%%%%%%%%%%%%%%%%%%%%%%%%%%%%%%%%%%%%%%%%%%%
\newpage
\appendix
\onecolumn
\section{Prompts for Memory Extraction}
\label{app:prompts}

The following prompts were used to extract memories from conversation data:

\begin{figure*}[h]
    \centering
    \setlength{\fboxrule}{0.6pt}
    \fbox{%
        \begin{minipage}{0.96\textwidth}
        \footnotesize
        \raggedright
        \setlength{\parskip}{1pt}
        \setlength{\baselineskip}{0.95\baselineskip}
        \vspace{1pt}

        \texttt{You are an expert conversation segmentation specialist. Your goal is to analyze a series of messages in a conversation and segment them into coherent topical episodes.}\\[6pt]

        \texttt{\# TASK}\\
        \texttt{Read the conversation carefully and identify points where the topic shifts significantly.}\\
        \texttt{Group messages discussing a similar subject, event, or theme into a single episode.}\\[4pt]

        \texttt{An \textbf{episode} is defined as a sequence of messages that revolve around a core topic or theme.}\\
        \texttt{Your task is to segment the conversation into such episodes.}\\[4pt]

        \texttt{\# OUTPUT FORMAT}\\
        \texttt{Provide a JSON object with the following structure:}\\
        \texttt{\{}\\
        \texttt{    "episodes": [}\\
        \texttt{        \{}\\
        \texttt{            "topic": "<brief topic description>",}\\
        \texttt{            "indices": [<list of message indices in this episode>] }\\
        \texttt{        \},}\\
        \texttt{        ...}\\
        \texttt{    ]}\\
        \texttt{\}}\\[4pt]

        \texttt{Where each episode contains:}\\
        \texttt{- \textbf{topic}: A brief description (a few words) summarizing the main topic of the episode}\\
        \texttt{- \textbf{indices}: A list of 1-based indices of messages that belong to this episode}\\[4pt]

        \texttt{\# GUIDELINES}\\
        \texttt{1. \textbf{Segmentation Criteria}}\\
        \texttt{~~~ - Topical shift: Identify when a new subject, event, or theme is introduced.}\\
        \texttt{~~~ - Transitions: Look for phrases like "By the way", "Changing the subject", or "On another note".}\\
        \texttt{~~~ - Time gaps: Large time lapses may indicate a new episode.}\\
        \texttt{~~~ - Setting changes: Changes in speaker, location, or context can signal a new episode.}\\
        \texttt{~~~ - Topical grouping: Consecutive messages discussing the same topic belong to the same episode.}\\[4pt]

        \texttt{2. \textbf{Episode Length}}\\
        \texttt{~~~ - Typically 2--8 messages per episode.}\\
        \texttt{~~~ - Combine messages if they discuss the same topic.}\\
        \texttt{~~~ - Avoid episodes longer than 8 messages covering multiple sub-topics.}\\
        \texttt{~~~ - Do not treat a single message as an episode unless it clearly marks a shift.}\\
        \texttt{~~~ - When in doubt, split into smaller episodes.}\\[4pt]

        \texttt{3. \textbf{Formatting Rules}}\\
        \texttt{~~~ - Use 1-based indexing for message indices.}\\
        \texttt{~~~ - Include all messages exactly once (no gaps or overlaps).}\\
        \texttt{~~~ - Indices in each episode should be consecutive.}\\[4pt]

        \texttt{\# EXAMPLE OUTPUT}\\
        \texttt{... }\\[2pt]

        \texttt{\# CONVERSATION TO SEGMENT}\\
        \texttt{\{messages\}}\\[2pt]

        \end{minipage}
    }
    \captionsetup{labelformat=default, name=Table}
    \caption{Prompt for segmenting conversations into coherent episodic units.}
    \label{fig:conversation-segmentation-prompt}
\end{figure*}

\begin{figure*}[t]
    \centering
    \setlength{\fboxrule}{0.6pt}
    \fbox{%
        \begin{minipage}{0.96\textwidth}
        \footnotesize
        \raggedright
        \setlength{\parskip}{1pt}
        \setlength{\baselineskip}{0.95\baselineskip}
        \vspace{1pt}

        \texttt{You are an expert episodic memory generator that creates episodic memory summaries from conversation segments.}\\[6pt]

        \texttt{\# TASK}\\
        \texttt{Generate an episodic memory with an index and a detailed summary based on the provided conversation segment.}\\[4pt]

        \texttt{Use the following format:}\\
        \texttt{EpisodicIndex: [6--8 word summary capturing main topic, entity, or event]}\\
        \texttt{EpisodicValue: [1--3 sentences descriptive summary of the conversation]}\\[6pt]

        \texttt{\# GUIDELINES}\\
        \texttt{1. EpisodicIndex}\\
        \texttt{~~~ - Create a short index (6--8 words) capturing the main topic or event of the episode.}\\
        \texttt{~~~ - Include specific context (e.g., domain or entity) to avoid vagueness.}\\[2pt]

        \texttt{2. EpisodicValue}\\
        \texttt{~~~ - Generate 1--3 sentence summary capturing:}\\
        \texttt{~~~~~ * Main information of the conversation segment (topic, theme, or event).}\\
        \texttt{~~~~~ * Relevant participants, referred to by name if available.}\\
        \texttt{~~~~~ * Use original wording when possible.}\\
        \texttt{~~~ - Focus on ``what happened'' rather than specific granular details.}\\
        \texttt{~~~ - Make the summary self-contained and understandable without the original conversation.}\\
        \texttt{~~~ - Include visual content if images are present.}\\
        \texttt{~~~ - Use only information present in the conversation segment; do not add external knowledge or infer beyond the content.}\\[6pt]

        \texttt{\# INPUT}\\
        \texttt{\{content\}}\\[2pt]

        \texttt{\# OUTPUT}\\
        \texttt{Provide the episodic memory in the format specified above.}\\[2pt]

        \vspace{1pt}
        \end{minipage}
    }
    \captionsetup{labelformat=default, name=Table}
    \caption{Prompt for generating episodic memories from conversation segments.}
    \label{fig:episodic-memory-prompt}
\end{figure*}

\begin{figure*}[t]
    \centering
    \setlength{\fboxrule}{0.6pt}
    \fbox{%
        \begin{minipage}{0.96\textwidth}
        \footnotesize
        \raggedright
        \setlength{\parskip}{1pt}
        \setlength{\baselineskip}{0.95\baselineskip}
        \vspace{1pt}

        \texttt{You are an expert factual memory extraction assistant. Your goal is to extract factual memories from a conversation segment.}\\[6pt]

        \texttt{\# TASK}\\
        \texttt{Read the input conversation carefully and extract ALL factual memories that could be useful for future reference.}\\[4pt]

        \texttt{Produce each memory as a key-value pair in the following format:}\\
        \texttt{MemIndex: memory index for retrieval}\\
        \texttt{MemValue: memory value with all details supported directly from the given text.}\\[6pt]

        \texttt{\# GUIDELINES}\\
        \texttt{1. \textbf{Content and Scope}}\\
        \texttt{~~~ - Use only information explicitly mentioned in the conversation.}\\
        \texttt{~~~ - Capture ALL factual information that could be useful. When in doubt, create more rather than fewer memories.}\\
        \texttt{~~~ - Exclude greetings, small talk, or filler.}\\
        \texttt{~~~ - Split distinct facts into separate entries.}\\
        \texttt{~~~ - Include details about people, events, intentions, hobbies, preferences, states, beliefs, goals, future plans, times, and locations if mentioned.}\\
        \texttt{~~~ - Include visual content from images as textual context, integrating relevant facts naturally.}\\[4pt]

        \texttt{2. \textbf{Format and Style}}\\
        \texttt{~~~ - MemIndex: Short, human-readable, self-contained, unambiguous phrase. Include specific context (e.g., entity or domain) to avoid vagueness.}\\
        \texttt{~~~ - MemValue: One or two full factual sentences capturing all relevant details.}\\
        \texttt{~~~~~ * Use neutral and factual wording.}\\
        \texttt{~~~~~ * Use original wording from the conversation when possible.}\\
        \texttt{~~~~~ * Replace pronouns with specific names or entities for clarity.}\\
        \texttt{~~~~~ * Convert relative times/dates (e.g., ``yesterday'', ``next week'') to absolute dates based on the conversation timestamp.}\\[6pt]

        \texttt{Timestamp of conversation: \{timestamp\}}\\[2pt]

        \texttt{Input Conversation: \{content\}}\\[2pt]

        \texttt{\# OUTPUT}\\
        \texttt{Produce all factual memories in the format specified above.}\\[2pt]

        \vspace{1pt}
        \end{minipage}
    }
    \captionsetup{labelformat=default, name=Table}
    \caption{Prompt for extracting factual memories from conversation segments.}
    \label{fig:factual-memory-prompt}
\end{figure*}

\begin{figure*}[t]
    \centering
    \setlength{\fboxrule}{0.6pt}
    \fbox{%
        \begin{minipage}{0.96\textwidth}
        \footnotesize
        \raggedright
        \setlength{\parskip}{1pt}
        \setlength{\baselineskip}{0.95\baselineskip}
        \vspace{1pt}

        \texttt{You are a memory management assistant. Given a new memory entry and similar existing entries, determine whether to update an existing entry or add a new one.}\\[6pt]

        \texttt{NEW MEMORY ENTRY:}\\
        \texttt{Index: \{new\_index\}}\\
        \texttt{Value: \{new\_value\}}\\[6pt]

        \texttt{EXISTING SIMILAR ENTRIES:}\\
        \texttt{\{candidates\_info\}}\\[6pt]

        \texttt{INSTRUCTIONS:}\\
        \texttt{1. Analyze if the new entry should update any existing entry based on semantic similarity and content overlap.}\\
        \texttt{2. If an update is needed, determine which candidate entry is best to update.}\\
        \texttt{3. Generate an updated memory value that combines relevant information from both entries.}\\
        \texttt{4. Decide whether the memory index should be updated to better reflect the combined information.}\\[4pt]

        \vspace{1pt}
        \end{minipage}
    }
    \captionsetup{labelformat=default, name=Table}
    \caption{Prompt for deciding whether to update an existing memory entry or create a new one.}
    \label{fig:memory_update_prompt}
\end{figure*}

\begin{figure*}[t]
    \centering
    \setlength{\fboxrule}{0.6pt}
    \fbox{%
        \begin{minipage}{0.96\textwidth}
        \footnotesize
        \raggedright
        \setlength{\parskip}{1pt}
        \setlength{\baselineskip}{0.95\baselineskip}
        \vspace{1pt}

        \texttt{You are a memory-indexing assistant optimized for knowledge retrieval. Your goal is to create cue indices that serve as semantic anchors for specific memories.}\\[6pt]

        \texttt{\# TASK}\\
        \texttt{For each memory provided, generate 1--3 short, meaningful \textbf{CUE ANCHORS} that can later help recall or reason about that memory.}\\
        \texttt{Provide the cue anchors as a list of strings for each memory.}\\[6pt]

        \texttt{\# GUIDELINES}\\
        \texttt{1. \textbf{Definition}: A cue anchor is a concise phrase (2--4 words) that anchors a specific topic to a memory.}\\
        \texttt{~~~ It follows the structure: [Main Entity] + [Key Aspect].}\\
        \texttt{~~~ - \textbf{Main Entity}: the primary person, domain, or object involved (the ``who'' or ``what'').}\\
        \texttt{~~~ - \textbf{Key Aspect}: the associated event, preference, action, state, or object.}\\[4pt]

        \texttt{~~~ Example patterns:}\\
        \texttt{~~~ - [Person] [Event/Activity] $\rightarrow$ ``Jane hiking trip'', ``Mike vacation''}\\
        \texttt{~~~ - [Person] [Hobby/Preference] $\rightarrow$ ``Michael jazz music'', ``Sophie vegan diet''}\\
        \texttt{~~~ - [Person] [Condition/State] $\rightarrow$ ``Emma career change'', ``Liam health problems''}\\
        \texttt{~~~ - [Person] [Object/Relation] $\rightarrow$ ``Alice research paper'', ``David guitar''}\\
        \texttt{~~~ - [Domain] [Attribute/Artifact] $\rightarrow$ ``Project Orion timeline'', ``Product X features''}\\[6pt]

        \texttt{2. \textbf{Specificity}: Avoid generic single words (e.g., ``summer'', ``happiness'', ``project meeting'').}\\
        \texttt{~~~ Every cue anchor must be contextually anchored to a main entity mentioned in the memory.}\\
        \texttt{~~~ Use concrete aspects (e.g., ``Mike mental health problems'' rather than ``Mike feelings'').}\\[4pt]

        \texttt{3. \textbf{Atomicity}: Each cue index should capture a single, indivisible aspect.}\\
        \texttt{~~~ Do not include timestamps, exact numbers, or multiple descriptors.}\\
        \texttt{~~~ Prefer generalizable cues (e.g., ``Mike birthday party'' over ``Mike birthday party 2023'').}\\[4pt]

        \texttt{4. \textbf{Distinct Facets}: A memory may have multiple cue indices, each targeting a different dimension.}\\
        \texttt{~~~ Cue indices for the same memory should not overlap in meaning.}\\
        \texttt{~~~ Avoid near-duplicates (e.g., ``Project Phoenix kickoff'' vs.\ ``Project Phoenix launch'').}\\[4pt]

        \texttt{5. \textbf{Uniqueness}: Do not repeat the primary memory index as a cue index.}\\[2pt]

        \texttt{6. \textbf{Purpose}: Cue indices provide additional semantic keys beyond the primary index,}\\
        \texttt{~~~ enabling recall, reasoning, and linking of related memories.}\\[6pt]

        \texttt{\# EXAMPLES}\\
        \texttt{Primary Abstraction: ``Jane's hiking trip to Appalachian Trail''}\\
        \texttt{Memory Value: ``Last summer, Jane went on a week-long hiking trip along the Appalachian Trail. She enjoyed the scenic views and challenging trails.''}\\
        \texttt{Cue Anchors: [``Jane hiking'', ``Appalachian Trail views'', ``Jane summer trip'']}\\[4pt]

        \texttt{Primary Abstraction: ``Mike's surprise birthday party''}\\
        \texttt{Memory Value: ``Mike's friends organized a surprise birthday party for him at his favorite restaurant Bistro Max.''}\\
        \texttt{Cue Anchors: [``Mike birthday party'', ``Mike favorite restaurant'', ``Mike friends gathering'']}\\[4pt]

        \texttt{Primary Abstraction: ``Project Orion launch delay''}\\
        \texttt{Memory Value: ``The launch of Project Orion has been delayed due to unforeseen technical issues that need to be resolved.''}\\
        \texttt{Cue Anchors: [``Project Orion launch'', ``Project Orion technical issues'']}\\[4pt]

        \texttt{Primary Abstraction: ``Emma went swimming''}\\
        \texttt{Memory Value: ``Emma went swimming during her vacation.''}\\
        \texttt{Cue Anchors: [``Emma swimming'']}\\[8pt]

        \texttt{\# MEMORIES TO PROCESS}\\
        \texttt{\{memories\}}\\[2pt]

        \vspace{1pt}
        \end{minipage}
    }
    \captionsetup{labelformat=default, name=Table}
    \caption{Prompt for generating cue indices as semantic anchors for memory retrieval.}
    \label{fig:cue-index-prompt-appendix}
\end{figure*}

\clearpage

\section{Evaluation Setup}
\label{app:eval_setup}

Following prior work, we adopt the same evaluation protocol for LLM-as-a-judge scoring from prior work. Specifically, for LoCoMo, we use \texttt{ANSWER\_PROMPT} from the official Mem0 GitHub repository \url{https://github.com/mem0ai/mem0/blob/main/evaluation/prompts.py} for answer generation, and \url{https://github.com/mem0ai/mem0/blob/main/evaluation/metrics/llm_judge.py} for LLM-as-a-judge scoring. 

For LongMemEval, we use the evaluation prompt provided in the official GitHub repository \url{https://github.com/xiaowu0162/LongMemEval/blob/main/src/evaluation/evaluate_qa.py} for LLM-as-a-judge scoring.

To ensure a fair comparison, we employ \texttt{gpt-4o-mini} as the evaluation model across all experiments, consistent with prior work. Additionally, we fix the random seed to 42 for reproducibility.

For latency evaluation, we use a compute instance located in East US (32 cores, 128 GB RAM, 256 GB disk) and query an Azure OpenAI endpoint located in Sweden Central.

\section{Preference-based Group-Relative Policy Updates.}
\subsection{Motivation.}
In sequential memory retrieval, step-level rewards are often noisy or unavailable, while the true objective—such as answer quality, grounding, and efficiency—is typically observable only after completing an entire retrieval trajectory.
Preference-based learning avoids explicit per-step supervision by comparing multiple retrieval trajectories generated for the same query and updating the policy to favor higher-quality trajectories.

\subsection{Trajectory Generation.}
Given a query $q$, we sample a group of $G$ retrieval trajectories:
\begin{equation}
\{\tau^{(g)}\}_{g=1}^{G}, \quad
\tau^{(g)} = \{(s_t^{(g)}, a_t^{(g)})\}_{t=0}^{T_g},
\end{equation}
using the current policy $\pi_\theta$, or a mixture with a reference policy for exploration.
Each trajectory produces a retrieved memory set $\mathcal{W}^{(g)}$.

\subsection{Judge-Based Trajectory Scoring.}
Each trajectory is evaluated by a judge that outputs a scalar score reflecting retrieval quality.
The judge may be implemented as a lightweight learned model, a frozen LLM-based evaluator, or a deterministic heuristic.
The trajectory score is decomposed into the following components.

\emph{Groundedness.}
Groundedness measures whether the final answer or reasoning is supported by the retrieved memories:
\begin{equation}
\mathrm{Ground}(\tau)
=
\textsc{Judge}_{\mathrm{ground}}(q, \mathcal{W}).
\end{equation}
This term can be instantiated using LLM-based judgments of evidence support or heuristic measures such as entailment or citation coverage.

\emph{Redundancy.}
Redundancy penalizes repeated or highly overlapping memories:
\begin{equation}
\mathrm{Redund}(\tau)
=
\frac{1}{|\mathcal{W}|^2}
\sum_{m_i, m_j \in \mathcal{W}}
\mathbb{I}\!\left[\mathrm{sim}(m_i, m_j) > \delta\right].
\end{equation}

\emph{Cost.}
Cost accounts for retrieval budget consumption:
\begin{equation}
\mathrm{Cost}(\tau)
=
\sum_t \mathrm{Cost}(a_t).
\end{equation}

\subsection{Scalar Trajectory Score.}
The judge aggregates the above components into a single trajectory-level score:
\begin{equation}
J(\tau)
=
w_1 \cdot \mathrm{Ground}(\tau)
-
w_2 \cdot \mathrm{Redund}(\tau)
-
w_3 \cdot \mathrm{Cost}(\tau).
\end{equation}
This score is defined at the trajectory level and does not require step-wise annotations.

\subsection{Group-Relative Advantage.}
Rather than relying on absolute scores, which may be noisy or query-dependent, we compute group-relative advantages within each query group:
\begin{equation}
\tilde{A}^{(g)}
=
J(\tau^{(g)})
-
\frac{1}{G}
\sum_{g'=1}^{G}
J(\tau^{(g')}).
\end{equation}
This normalization yields zero-mean advantages within each group, improving robustness to judge bias and score scaling while encouraging relative improvement.

\subsection{Policy Update.}
The policy is updated to increase the likelihood of actions from trajectories with positive relative advantage:
\begin{equation}
\mathcal{L}_{\mathrm{GR}}(\theta)
=
-
\sum_{g=1}^{G}
\tilde{A}^{(g)}
\sum_t
\log \pi_\theta\!\left(a_t^{(g)} \mid s_t^{(g)}\right).
\end{equation}
To prevent policy drift, we optionally add KL regularization with respect to a reference policy $\pi_{\mathrm{ref}}$:
\begin{equation}
\mathcal{L}(\theta)
=
\mathcal{L}_{\mathrm{GR}}(\theta)
+
\beta
\sum_t
\mathrm{KL}\!\left(
\pi_\theta(\cdot \mid s_t)
\;\|\;
\pi_{\mathrm{ref}}(\cdot \mid s_t)
\right).
\end{equation}

\section{A Unifying Theory of Structured Memory Retrieval}
\label{appendix:theory}

\subsection{Preliminaries and Notation}

We briefly summarize the minimal notation required for theoretical analysis, relying on the definitions introduced in the Method section.

Let $\mathcal{M}$ denote the set of memory entries maintained by the system. Each memory entry is associated with a unique \emph{primary abstraction} and a (possibly empty) set of \emph{cue anchors}. We denote the primary abstraction space by $\mathcal{A}$ and the cue anchor space by $\mathcal{C}$.

The memory structure is characterized by two assignment relations:
\[
\alpha:\mathcal{M}\rightarrow\mathcal{A}, \qquad 
\Gamma:\mathcal{M}\rightarrow 2^{\mathcal{C}},
\]
where $\alpha(m)$ assigns each memory entry $m$ to exactly one primary abstraction, and $\Gamma(m)$ returns the set of cue anchors associated with $m$. These relations induce abstraction–memory and cue–memory associations, which together define the indexing structure over $\mathcal{M}$.

Given a query $q$, the system scores abstractions and cue anchors using query-dependent scoring functions
\[
s_A(q,a), \qquad s_C(q,c),
\]
and selects a bounded set of top-ranked abstractions and cues. Retrieval is then defined structurally as the union of memory entries supported by the selected abstractions and cue anchors. This abstraction-and-cue–based retrieval operator constitutes the core retrieval mechanism analyzed in this section.

To support multi-hop and graph-style retrieval, memory entries may additionally be connected through traversal relations induced by shared cue anchors or other structural links. Let $\mathcal{R}_L(q)$ denote the result of applying up to $L$ traversal steps starting from the initial retrieval set. Setting $L=0$ recovers single-step retrieval, while larger $L$ enables iterative expansion analogous to graph neighborhood search.

\subsection{Traditional RAG and KG Retrieval as Special Cases}

We show that both \emph{traditional RAG} and \emph{knowledge-graph (KG) retrieval} can be expressed as special cases of Memora by choosing appropriate key spaces and relations.

\begin{theorem}[Flat RAG as a Special Case of Memora]
\label{thm:rag-special-case}
Let $\mathcal{D}$ be a corpus and let $\mathcal{S}(\cdot)$ be a segmentation function that produces a set of chunks (segments). Consider a flat RAG retriever that, for any query $q$, returns
\begin{equation}
\mathcal{R}_{\mathrm{RAG}}(q)
=
\operatorname{TopK}_{s \in \bigcup_{d \in \mathcal{D}} \mathcal{S}(d)}
\mathrm{sim}(q, s),
\end{equation}
where $\mathrm{sim}$ is a similarity function over chunk representations. Then there exists a Memora instantiation and a policy $\pi$ such that the retrieval set returned by Algorithm~\ref{alg:policy_retrieval} equals $\mathcal{R}_{\mathrm{RAG}}(q)$ for all queries $q$.
\end{theorem}

\begin{proof}
Define the Memora memory corpus by taking each chunk $s$ as one memory entry,
\begin{equation}
\mathcal{M} \;=\; \{ m(s) \mid s \in \bigcup_{d \in \mathcal{D}} \mathcal{S}(d)\},
\qquad
m(s) = (a(s), v(s), \mu(s)).
\end{equation}
Let the primary abstraction equal the memory content,
\begin{equation}
a(s) = v(s) = s,
\end{equation}
and let the cue-anchor set be empty for every entry,
\begin{equation}
\mathcal{C}(m(s)) = \emptyset.
\end{equation}

Consider the restricted action space $\mathcal{A}=\{\textsc{QueryA},\textsc{Stop}\}$ and define the retrieval primitive \textsc{QueryA} to return the top-$k$ memory entries ranked by abstraction similarity,
\begin{equation}
\textsc{QueryA}(q)
=
\operatorname{TopK}_{m(s)\in \mathcal{M}}
\mathrm{sim}\bigl(q, a(s)\bigr)
=
\operatorname{TopK}_{s \in \bigcup_{d \in \mathcal{D}} \mathcal{S}(d)}
\mathrm{sim}(q, s).
\end{equation}
Let the policy $\pi$ choose \textsc{QueryA} at $t=0$ and then \textsc{Stop}:
\begin{equation}
\pi(a_0=\textsc{QueryA}\mid s_0)=1,
\qquad
\pi(a_1=\textsc{Stop}\mid s_1)=1.
\end{equation}
Algorithm~\ref{alg:policy_retrieval} therefore terminates after one retrieval step and returns
\begin{equation}
\mathcal{W}_1 = \textsc{QueryA}(q) = \mathcal{R}_{\mathrm{RAG}}(q),
\end{equation}
which proves the claim.
\end{proof}

This theorem shows that flat chunk-based RAG corresponds to a degenerate configuration of Memora in which each segment forms a single memory entry, abstractions coincide with raw memory content, cue anchors are unused, and retrieval reduces to a single abstraction query step.

\subsubsection{Knowledge Graph Retrieval}

We analyze the relationship between Memora and KG-based retrieval under two settings: 
(i) implicit KGs, where neighborhood structure is induced by semantic similarity, and 
(ii) explicit KGs, where symbolic relations are available. Both can be expressed within the Memora framework using cue anchors and traversal actions.

\paragraph{Implicit KG retrieval.}
We first consider KG-style retrieval without explicit relational edges. Let $\mathcal{M}$ be the memory corpus and let
\[
\pi : \mathcal{M} \rightarrow V
\]
associate each memory entry with an entity $v \in V$. Given a query $q$, an implicit KG retriever selects a seed entity set $S(q) \subseteq V$ and retrieves memories attached to entities reachable within $L$ steps under a similarity-induced neighborhood relation.

Formally, define an implicit entity adjacency
\[
v \sim v' \iff \mathrm{sim}(v, v') \ge \delta,
\]
and let $\mathsf{Nbr}^{\mathrm{imp}}_L(S(q))$ denote the $L$-hop neighborhood under this relation. The retrieval result is
\[
\mathcal{R}_{\mathrm{KG}}^{\mathrm{imp}}(q)
=
\{ m \in \mathcal{M} \mid \pi(m) \in \mathsf{Nbr}^{\mathrm{imp}}_L(S(q)) \}.
\]

\begin{theorem}[Implicit KG Retrieval as a Special Case of Memora]
\label{thm:implicit-kg}
For any implicit KG retriever $\mathcal{R}_{\mathrm{KG}}^{\mathrm{imp}}(q)$, there exists a Memora instantiation and traversal depth $L$ such that $\mathcal{R}_L(q)=\mathcal{R}_{\mathrm{KG}}^{\mathrm{imp}}(q)$ for all queries $q$.
\end{theorem}

\begin{proof}
Let the cue anchor space be $\mathcal{C} := V$, and associate each memory entry with exactly one cue anchor:
\[
\Gamma(m) := \{\pi(m)\}, \qquad \forall m \in \mathcal{M}.
\]
Let the primary abstraction space be trivial so that abstraction-based retrieval does not affect the result.

Define cue scoring such that
\[
\operatorname{TopK}_{c \in \mathcal{C}} s_C(q,c) = S(q),
\]
yielding the initial retrieval
\[
\mathcal{R}_0(q) = \{ m \in \mathcal{M} \mid \pi(m) \in S(q) \}.
\]
Define cue–cue traversal in Memora using the same similarity relation:
\[
c \leadsto c' \iff \mathrm{sim}(c, c') \ge \delta.
\]
Applying $L$ traversal steps retrieves exactly those memory entries whose associated cues lie in $\mathsf{Nbr}^{\mathrm{imp}}_L(S(q))$, hence
\[
\mathcal{R}_L(q) = \mathcal{R}_{\mathrm{KG}}^{\mathrm{imp}}(q).
\]
\end{proof}

\paragraph{Explicit KG retrieval.}
We now consider traditional knowledge-graph retrieval with explicit symbolic relations. Let $G=(V,E)$ be a knowledge graph, where $V$ denotes entities and $E$ denotes typed relations between entities. Each memory entry is attached to a graph element through a mapping
\[
\pi : \mathcal{M} \rightarrow V \cup E.
\]
Given a query $q$, KG retrieval selects a seed set $S(q) \subseteq V \cup E$ and retrieves memory entries associated with elements in the $L$-hop graph neighborhood:
\[
\mathcal{R}_{\mathrm{KG}}^{\mathrm{exp}}(q)
=
\{ m \in \mathcal{M} \mid \pi(m) \in \mathsf{Nbr}_L(S(q)) \}.
\]

\begin{theorem}[Explicit KG Retrieval as an Extended Case of Memora]
\label{thm:explicit-kg}
For any explicit KG retriever $\mathcal{R}_{\mathrm{KG}}^{\mathrm{exp}}(q)$, there exists an \emph{extended} Memora instantiation such that the multi-hop retrieval result $\mathcal{R}_L(q)$ produced by Memora equals $\mathcal{R}_{\mathrm{KG}}^{\mathrm{exp}}(q)$ for all queries $q$.
\end{theorem}

\begin{proof}
Consider an extended Memora configuration in which cue anchors explicitly encode KG entities and relations by setting
\[
\mathcal{C} := V \cup E,
\qquad
\Gamma(m) := \{\pi(m)\}, \ \forall m \in \mathcal{M}.
\]
In addition, Memora is augmented with a cue–cue traversal relation that exactly mirrors the KG structure:
\[
c \leadsto c' \iff (c,c') \in E.
\]
This extension requires Memora to adopt the same relational assumptions as the underlying KG, namely that edges are explicitly defined and traversable.

Seed selection is performed through cue scoring such that
\[
\operatorname{TopK}_{c \in \mathcal{C}} s_C(q,c) = S(q),
\]
yielding the initial retrieval
\[
\mathcal{R}_0(q) = \{ m \in \mathcal{M} \mid \pi(m) \in S(q) \}.
\]
Since cue–cue traversal coincides exactly with KG edges, applying $L$ traversal steps in Memora recovers the same $L$-hop neighborhood as $\mathsf{Nbr}_L(S(q))$, and therefore
\[
\mathcal{R}_L(q)=\mathcal{R}_{\mathrm{KG}}^{\mathrm{exp}}(q).
\]
\end{proof}

\paragraph{Interpretation.}
Explicit KG retrieval corresponds to an extended instantiation of Memora in which cue anchors are restricted to symbolic entities and relations, and traversal operations are constrained to follow predefined KG edges. This setting recovers classical KG behavior but requires Memora to inherit the same structural assumptions and construction costs as the KG. In contrast, the implicit KG case arises naturally within the base Memora design, where cue anchors and traversal relations can be learned or induced without explicit symbolic graphs.

\subsection{Memora as a Strict Generalization: Expressivity}

The special-case results above establish that flat RAG-style retrieval and
KG-style seed-and-expand retrieval can be realized within Memora under
suitable parameterizations. We next formalize a strictness result showing
that Memora can represent retrieval behaviors that are not realizable by
(i) flat top-$k$ similarity retrieval over raw memory content and
(ii) KG retrievers with a fixed single-attachment structure, under standard
structural constraints.

\begin{definition}[Retrieval classes]
A retrieval function maps a query to a subset of memory entries,
\[
\mathcal{R}:\mathcal{Q}\to 2^{\mathcal{M}}.
\]
We consider the following three retrieval classes.

\begin{enumerate}
\item \textbf{Flat top-$k$ similarity retrieval.}
There exists a single scoring function $s(q,m)$ such that, for every query $q$,
\[
\mathcal{R}(q)
=
\operatorname{TopK}_{k}\!\left(
\{\, s(q,m) : m \in \mathcal{M} \,\}
\right),
\]
and therefore $|\mathcal{R}(q)| = k$.

\item \textbf{KG seed-and-expand retrieval with fixed attachment.}
There exists a fixed attachment map $\pi:\mathcal{M}\to V$ on a fixed graph
$G=(V,E)$ such that, for every query $q$,
\[
\mathcal{R}(q)
=
\{\, m\in\mathcal{M} : \pi(m)\in \mathsf{Nbr}_L(S(q)) \,\},
\]
where $S(q)\subseteq V$ is a query-dependent seed set and
$\mathsf{Nbr}_L(\cdot)$ denotes the $L$-hop neighborhood operator.

\item \textbf{Memora retrieval.}
The retrieval function is realizable by Memora using primary abstractions
$\alpha(m)$ and cue anchors $\Gamma(m)$, including the gated form
\begin{equation}
\label{eq:gated-Memora}
\mathcal{R}_{\cap}(q)
:=
\{ m\in\mathcal{M} : \alpha(m)\in A_q \}
\;\cap\;
\{ m\in\mathcal{M} : \Gamma(m)\cap C_q \neq \emptyset \},
\end{equation}
where
\[
A_q=\operatorname{TopK}_{K_A}\!\left(\{\, s_A(q,a):a\in\mathcal{A} \,\}\right),
\qquad
C_q=\operatorname{TopK}_{K_C}\!\left(\{\, s_C(q,c):c\in\mathcal{C} \,\}\right).
\]
\end{enumerate}
\end{definition}

\begin{theorem}[Strictness under mixed-key constraints]
\label{thm:strictness-formal}
There exists a Memora retrieval function $\mathcal{R}^\star$ such that, for
any fixed $k$ and any fixed $L$, $\mathcal{R}^\star$ cannot be realized by
flat top-$k$ similarity retrieval and cannot be realized by KG
seed-and-expand retrieval with a fixed single-attachment map.
\end{theorem}

\begin{proof}
We prove the theorem by giving an explicit construction of a retrieval function realizable by Memora but not by fixed top-$k$ or fixed-attachment KG retrieval. The construction targets a retrieval behavior defined by the joint enforcement of two constraints: a coarse structural restriction induced by
a primary abstraction, and a fine-grained selector induced by a cue anchor. Memora can realize such mixed-key constraints through intersection across its indexing spaces, whereas flat top-$k$ retrievers and KG retrievers with fixed single-attachment are inherently unable to represent this joint selection under the stated constraints.

\paragraph{Step 1: Mixed-key target.}
Partition the memory corpus using two primary abstractions
$\mathcal{A}=\{a^{(1)},a^{(2)}\}$ and define
\[
\mathcal{M}^{(1)} := \{m\in\mathcal{M}:\alpha(m)=a^{(1)}\},
\qquad
\mathcal{M}^{(2)} := \{m\in\mathcal{M}:\alpha(m)=a^{(2)}\}.
\]
Fix a cue anchor $c^\star\in\mathcal{C}$ appearing in both groups, and let
\[
\mathcal{N}^{(1)} := \{m\in\mathcal{M}^{(1)}: c^\star\in\Gamma(m)\},
\qquad
\mathcal{N}^{(2)} := \{m\in\mathcal{M}^{(2)}: c^\star\in\Gamma(m)\}.
\]
Define the target retrieval function
\[
\mathcal{R}^\star(q) := \mathcal{N}^{(1)},
\]
and assume $|\mathcal{N}^{(1)}|>k$.

\paragraph{Step 2: Realizability within Memora.}
We show that the target retrieval function $\mathcal{R}^\star$ is realizable
by Memora. By definition, Memora supports retrieval predicates formed by
the intersection of abstraction-level selection and cue-level selection.
Consider the abstraction set $A_q=\{a^{(1)}\}$ and the cue set
$C_q=\{c^\star\}$. Since Memora allows independent query-conditioned
selection over the abstraction space $\mathcal{A}$ and the cue space
$\mathcal{C}$, there exist scoring functions $s_A$ and $s_C$ and finite
cutoffs $K_A,K_C$ such that these sets are selected.

Substituting these sets into the gated retrieval operator in
Eq.~\eqref{eq:gated-Memora} yields
\[
\mathcal{R}_{\cap}(q)
=
\{m\in\mathcal{M}:\alpha(m)=a^{(1)}\}
\cap
\{m\in\mathcal{M}:c^\star\in\Gamma(m)\}
=
\mathcal{R}^\star(q).
\]
Thus $\mathcal{R}^\star$ lies within the class of retrieval functions
realizable by Memora.

\paragraph{Step 3: Impossibility for flat top-$k$ similarity retrieval.}
We show that $\mathcal{R}^\star$ cannot be realized by any flat top-$k$
similarity retriever. By definition, any such retriever is induced by a
single real-valued scoring function $s:\mathcal{Q}\times\mathcal{M}\to\mathbb{R}$
and returns exactly $k$ memory entries for every query:
\begin{equation}
\label{eq:rag-cardinality}
|\mathcal{R}(q)| = k, \qquad \forall q\in\mathcal{Q}.
\end{equation}

In contrast, the target retrieval function $\mathcal{R}^\star$ is defined as
\[
\mathcal{R}^\star(q) = \mathcal{N}^{(1)},
\]
where $|\mathcal{N}^{(1)}|>k$ by construction. Therefore,
\begin{equation}
\label{eq:rstar-cardinality}
|\mathcal{R}^\star(q)| > k.
\end{equation}
Equations~\eqref{eq:rag-cardinality} and~\eqref{eq:rstar-cardinality}
immediately yield a contradiction: no flat top-$k$ retriever can reproduce
the output of $\mathcal{R}^\star$ for this query.

More fundamentally, flat similarity retrieval induces a total preorder over
$\mathcal{M}$ via $s(q,\cdot)$ and selects a prefix of fixed length $k$.
Any retrieval predicate whose extension is not expressible as such a fixed
prefix—independent of internal structure or semantic grouping—lies outside
the expressive scope of flat top-$k$ retrieval. Hence $\mathcal{R}^\star$
is not realizable by flat similarity ranking.

\paragraph{Step 4: Impossibility for KG retrieval with fixed single-attachment.}
We now show that $\mathcal{R}^\star$ cannot be realized by KG-style
seed-and-expand retrieval under a fixed single-attachment map
$\pi:\mathcal{M}\to V$.

For any such retriever, the retrieved set has the form
\begin{equation}
\label{eq:kg-form}
\mathcal{R}(q)
=
\{ m\in\mathcal{M} : \pi(m)\in \mathsf{Nbr}_L(S(q)) \},
\end{equation}
where $\mathsf{Nbr}_L(\cdot)$ denotes $L$-hop graph neighborhoods. Crucially,
membership in $\mathcal{R}(q)$ depends \emph{only} on the attachment $\pi(m)$
and the graph structure, and is therefore invariant to any memory attributes
not encoded in $\pi(m)$.

In the constructed instance, memories in $\mathcal{N}^{(1)}$ and
$\mathcal{N}^{(2)}$ share the same cue anchor $c^\star$ but differ in their
primary abstractions $\alpha(m)$. Since $\pi$ is fixed and single-valued,
it cannot simultaneously encode both abstraction-level information and
cue-level information without collapsing distinct semantic dimensions.
As a result, there exist memories
\[
m_1\in\mathcal{N}^{(1)}, \quad m_2\in\mathcal{N}^{(2)}
\]
such that $\pi(m_1)$ and $\pi(m_2)$ lie in the same or overlapping graph
neighborhoods whenever the cue signal $c^\star$ is reachable.

Consequently, any seed set $S(q)$ and radius $L$ for which
\[
m_1 \in \mathcal{R}(q)
\]
necessarily implies
\[
m_2 \in \mathcal{R}(q),
\]
unless the graph or attachment map explicitly encodes the abstraction
partition. This contradicts the definition of $\mathcal{R}^\star$, which
selects $\mathcal{N}^{(1)}$ while excluding $\mathcal{N}^{(2)}$.

Therefore, under a fixed single-attachment structure, KG neighborhood
expansion cannot enforce the joint predicate
\[
\alpha(m)=a^{(1)} \;\wedge\; c^\star\in\Gamma(m),
\]
and $\mathcal{R}^\star$ is not realizable by KG retrieval.

\paragraph{Remark.}
The strictness result follows from mixed-key constraints: Memora can jointly
enforce coarse structural scope through primary abstractions and fine-grained
selection through cue anchors, a capability unavailable to flat similarity
retrieval and KG retrieval with fixed single-attachment under standard
assumptions.
\end{proof}

\clearpage
\begin{theorem}[Efficiency gain from abstraction-first + cue-anchor ANN retrieval]
\label{thm:candidate-reduction}
Assume that memories are partitioned into abstractions with expected bucket size $B$,
so that $|\mathcal{A}|\approx N/B$, and that each abstraction has on average $m$ cue
anchors, yielding a total of $m|\mathcal{A}|\approx mN/B$ cue anchors indexed for
retrieval. Under the variant in which query-time retrieval is performed via
(1) an ANN lookup over abstractions and (2) an ANN lookup over cue anchors, without
explicit intra-abstraction enumeration, the expected query-time cost of Memora
satisfies
\[
T_{\mathrm{Harmo}}(q)
=
O\!\left(
\log\!\left(\frac{mN^2}{B^2}\right)
\right).
\]
In contrast, a flat ANN-based retriever that indexes all $N$ memories incurs $T_{\mathrm{RAG}}(q)
=
O(\log N)$
under the same index-family assumptions. Consequently, abstraction-first retrieval
yields a multiplicative efficiency improvement of
\[
\Omega\!\left(
\frac{\log N}{2\log N+\log m-2\log B}
\right).
\]
\end{theorem}

\begin{proof}
We upper bound the query-time cost of Memora by decomposing retrieval into two
indexed lookups, and compare against a flat ANN baseline.

\paragraph{Stage 1: Abstraction selection.}
Memora queries an ANN index over abstractions $\mathcal{A}$. Since abstractions
partition $N$ memories into buckets of expected size $B$, we have
$|\mathcal{A}|\approx N/B$. Under standard ANN index assumptions, querying an
index of size $|\mathcal{A}|$ incurs expected cost
\[
O(\log|\mathcal{A}|)
=
O\!\left(\log\!\left(\frac{N}{B}\right)\right).
\]

\paragraph{Stage 2: Cue-anchor retrieval.}
Memora then performs an ANN query over the cue-anchor index. If each abstraction
has on average $m$ cue anchors, then the cue-anchor index size is
\[
|\mathcal{U}| \approx m|\mathcal{A}| \approx \frac{mN}{B}.
\]
Thus the cue-anchor query incurs expected cost
\[
O(\log|\mathcal{U}|)
=
O\!\left(\log\!\left(\frac{mN}{B}\right)\right).
\]
By assumption of this variant, retrieved cue anchors provide direct references
to associated memories, so there is no additional intra-abstraction enumeration term.

\paragraph{Total cost.}
Summing the two stages yields
\begin{align*}
T_{\mathrm{Harmo}}(q)
&=
O\!\left(\log\!\left(\frac{N}{B}\right)\right)
+
O\!\left(\log\!\left(\frac{mN}{B}\right)\right)\\
&=
O\!\left(\log\!\left(\frac{N}{B}\cdot\frac{mN}{B}\right)\right)\\
&=
O\!\left(\log\!\left(\frac{mN^2}{B^2}\right)\right).
\end{align*}

\paragraph{Flat ANN baseline.}
A flat ANN-based retriever performs a single ANN query over $N$ indexed memories,
incurring expected query time
\[
T_{\mathrm{RAG}}(q)=O(\log N)
\]
under the same index-family assumptions.

\paragraph{Improvement (normalized form).}
Define the multiplicative efficiency improvement as
$T_{\mathrm{RAG}}(q)/T_{\mathrm{Harmo}}(q)$. Substituting the bounds gives
\[
\frac{T_{\mathrm{RAG}}(q)}{T_{\mathrm{Harmo}}(q)}
=
\Omega\!\left(
\frac{\log N}{\log\!\left(\frac{mN^2}{B^2}\right)}
\right).
\]
Expanding the denominator,
\[
\log\!\left(\frac{mN^2}{B^2}\right)
=
2\log N + \log m - 2\log B,
\]
so
\[
\frac{T_{\mathrm{RAG}}(q)}{T_{\mathrm{Harmo}}(q)}
=
\Omega\!\left(
\frac{\log N}{2\log N+\log m-2\log B}
\right).
\]
which proves the stated improvement bound.
\end{proof}

\paragraph{Analysis.} Conditions under which the efficiency improvement exceeds unity.

Let
\[
\mathrm{Imp}(N,B,m)
\;=\;
\frac{T_{\mathrm{RAG}}(q)}{T_{\mathrm{Harmo}}(q)}
\approx
\frac{\log N}{\log\!\left(\frac{mN^2}{B^2}\right)}.
\]
A sufficient condition for $\mathrm{Imp}(N,B,m) > 1$ is
\[
\log N
>
\log\!\left(\frac{mN^2}{B^2}\right)
\quad\Longleftrightarrow\quad
N > \frac{mN^2}{B^2}
\quad\Longleftrightarrow\quad
B^2 > mN.
\]
Equivalently, in normalized form,
\[
\mathrm{Imp}(N,B,m) > 1
\quad\text{whenever}\quad
2+\frac{\log m}{\log N}-2\frac{\log B}{\log N} < 1
\;\Longleftrightarrow\;
\frac{\log B}{\log N} > \frac{1}{2}\left(1+\frac{\log m}{\log N}\right).
\]
In particular, if $m$ is constant (or grows subpolynomially) and $B=\Omega(N^{1/2+\epsilon})$
for any $\epsilon>0$, then $\mathrm{Imp}(N,B,m) > 1$ for sufficiently large $N$.

\paragraph{Remark.}
In typical memory systems, $B^2>mN$ is a strong requirement; when it does not hold,
both approaches remain logarithmic and the advantage of abstraction-first retrieval
should be interpreted primarily as a constant-factor gain due to operating over
smaller indices (and, in practice, fewer distance computations and better cache locality).

----------------------------------------------

\paragraph{Implication.}
Primary abstraction provides a principled \emph{search space factorization}:
the retrieval process first narrows the search space using stable,
coarse-grained concepts, and then applies cue anchors to recover
fine-grained precision within the selected regions. Flat RAG is recovered
as a degenerate case when $B=1$ (each memory forms its own abstraction) or
when abstraction selection is disabled. KG retrieval is recovered when cue
anchors correspond to symbolic graph elements and candidate expansion
follows graph adjacency, as established in
Theorems~\ref{thm:rag-special-case} and~\ref{thm:explicit-kg}.

\subsection{Summary}

The Memora framework defines a general class of structured retrieval mechanisms based on (i) canonical organization via primary abstraction and (ii) flexible access via cue anchors, optionally enhanced with multi-hop traversal.
We formally showed that:
(i) traditional RAG is a degenerate special case (identity cues, no abstraction),
(ii) KG retrieval is also a special case (symbolic cues + graph expansion),
and (iii) Memora can represent richer mixed-key retrieval behaviors while enabling principled efficiency improvements through abstraction-first scoping.

\section{Case Study}
\label{app:case_study}

In this section, we present case studies demonstrating how \textsc{Memora} achieves superior performance compared to Mem0 and RAG. To isolate the benefits of our memory structure, we utilize \textsc{Memora} with a standard semantic retriever, and showcase the factual memories, thereby highlighting how the harmonic representation itself enhances memory management and retrieval.

\subsection{Case Study 1}

In this example (Table \ref{tab:case1}), \textsc{Memora} demonstrates superior memory retrieval precision. This success is attributed to \textsc{Memora}'s index-value representation, which decouples the navigation layer from the raw data. While traditional RAG often suffers from semantic drift and Mem0 can lose granularity through over-summarization, \textsc{Memora}'s indices serve as a structured guide to the memory space. This allows the system to pinpoint specific entities while preserving the original richness and contextual meaning of the memory items.

\newcommand{\cmark}{\textcolor{green!50!black}{\ding{51}}} % Tick
\newcommand{\xmark}{\textcolor{red}{\ding{55}}}             % Cross

\begin{table}[h]
\centering
\small
\caption{Case 1 and answers generated from three systems}
\begin{tabular}{p{3.2cm} p{10cm}}
\toprule
Question &
What did Mel and her kids paint in their latest project in July 2023? \\
\midrule
Reference Answer &
A sunset with a palm tree \\
\midrule
RAG Answer & \xmark \space A rainbow flag mural \\
Mem0 Answer & \xmark \space A painting similar to their last one \\
\textsc{Memora} Answer & \cmark \space Sunset scene with a palm tree and flowers \\
\bottomrule
\label{tab:case1}
\end{tabular}

\end{table}

\begin{table}[ht]
\centering
\small
\renewcommand{\arraystretch}{1.5}
\caption{Comparative analysis of top memories retrieved for Case 1. (part 1)}

\begin{tabular}{@{} l p{10.5cm} @{} }
\toprule
\textbf{Method} & \textbf{Retrieved Memories / Contextual Evidence} \\
\midrule
\textsc{Memora} & \textbf{Recent painting by Melanie and kids} \\
& \textit{Value:} Melanie's latest painting with the kids is a sunset scene featuring a palm tree and vibrant flowers against a sunset sky. \\
& \footnotesize \textbf{Cues:} `Melanie sunset painting', `Palm tree art', `Kids vibrant flowers' \\
\cmidrule(lr){2-2}

& \textbf{Melanie's work in progress and summer plans} \\
& \textit{Value:} Melanie is currently working on a project and is doing her best to complete it, her kids are excited about summer break, and they are thinking about going camping next month. \\
& \footnotesize \textbf{Cues:} `Melanie current project', `Kids summer break', `Family camping plans' \\
\cmidrule(lr){2-2}

& \textbf{Melanie's kids enjoying pottery making} \\
& \textit{Value:} Melanie's kids loved making pottery and were very excited to get their hands dirty and create something with clay. \\
& \footnotesize \textbf{Cues:} `Kids pottery making', `Clay art activity' \\
\cmidrule(lr){2-2}

& \textbf{Melanie's recent family painting activity} \\
& \textit{Value:} Melanie and her kids have just finished another painting similar to their last one. \\
& \footnotesize \textbf{Cues:} `Melanie family painting', `Kids collaborative artwork' \\
\cmidrule(lr){2-2}

& \textbf{Melanie's creative projects with kids} \\
& \textit{Value:} Melanie engages in painting with kids, focusing especially on nature-inspired themes. \\
& \footnotesize \textbf{Cues:} `Melanie nature painting', `Kids art engagement' \\
\bottomrule
\end{tabular}
\label{tab:case1_memory_1}
\end{table}

\begin{table}[ht]
\centering
\small
\renewcommand{\arraystretch}{1.4}
\caption{Comparative analysis of top memories retrieved for Case 1. (continued)}

\begin{tabular}{@{} l p{10.5cm} @{} }
\toprule
\textbf{Method} & \textbf{Retrieved Memories / Context} \\
\midrule
Mem0 & Melanie and children recently did a painting project last weekend. \\
& Melanie and her kids recently finished another painting similar to their last one. \\
& Melanie and children enjoy painting together, especially nature-inspired art. \\
& Melanie has been painting to keep busy. \\
& Melanie took her kids to a pottery workshop last Friday. \\
& Melanie helped with a painting that highlights the beauty of nature. \\
& Melanie is feeling inspired by autumn and planning a few paintings. \\
\midrule
RAG & \textbf{Context Fragment:} 1:33 pm on 25 August, 2023 | Caroline: Finding a community where I'm accepted... Stuff like this mural are really special to me! \dots 1:33 pm on 25 August, 2023 | Melanie: Caroline, glad you found a supportive community! \dots 1:33 pm on 25 August, 2023 | Caroline: The rainbow flag mural is important to me as it reflects the courage and strength of the trans community. The eagle symbolizes freedom and pride... \dots 1:33 pm on 25 August, 2023 | Melanie: I'm in awe of your courage as a trans person. Have you made any more art lately? \\
\bottomrule
\end{tabular}
\label{tab:case1_memory_2}
\end{table}

\subsection{Case Study 2}

In this example (Table \ref{tab:case2}), \textsc{Memora} demonstrates a robust capacity for information aggregation, correctly synthesizing disparate facts, the initial ownership of a dog and cat, followed by the later addition of a second cat named Bailey, into a single, comprehensive response. Unlike baseline methods that often retrieve fragmented memory fragments or lose connections, \textsc{Memora} effectively links related entities across non-contiguous parts of the dialogue context.

\begin{table}[ht]
\centering
\small
\setlength{\tabcolsep}{6pt}
\caption{Case Study 2 and answers generated from three systems}

\begin{tabular}{p{3.2cm} p{10cm}}
\toprule
Question & What pets does Melanie have? \\
\midrule
Reference Answer & Two cats and a dog \\
\midrule
RAG Answer & \xmark \space A cat named Oliver and another cat named Bailey \\
Mem0 Answer & \xmark \space Luna and Oliver \\
\textsc{Memora} Answer & \cmark \space Dog and two cats (Luna, Oliver, Bailey) \\
\bottomrule
\end{tabular}
\label{tab:case2}
\end{table}

\begin{table}[ht]
\centering
\small
\renewcommand{\arraystretch}{1.5}
\caption{Comparison of top retrieved memories for Case Study 2. (part 1)}

\begin{tabular}{@{} l p{10.5cm} @{} }
\toprule
\textbf{Method} & \textbf{Retrieved Memories / Contextual Evidence} \\
\midrule
\textsc{Memora} & \textbf{Melanie's pets and inquiries about pets} \\
& \textit{Value:} Melanie has a dog and a cat as pets. Melanie has two pets named Luna and Oliver. Melanie asked Caroline if she has any pets during their conversation. \\
& \footnotesize \textbf{Cues:} `Melanie pets', `Melanie pet names', `Melanie conversation about pets' \\
\cmidrule(lr){2-2}

& \textbf{Melanie's agreement on pets} \\
& \textit{Value:} Melanie agrees that pets bring joy and comfort. \\
& \footnotesize \textbf{Cues:} `Melanie pets joy', `Melanie pets comfort' \\
\cmidrule(lr){2-2}

& \textbf{Melanie's pets behavior} \\
& \textit{Value:} Melanie's pets, Luna and Oliver, are described as sweet and playful and they really liven up the house. \\
& \footnotesize \textbf{Cues:} `Melanie pets behavior', `Luna playful', `Oliver playful' \\
\cmidrule(lr){2-2}

& \textbf{Pets' effect on Melanie's family} \\
& \textit{Value:} Melanie states that their dog and cat brighten up their day and always make them smile. \\
& \footnotesize \textbf{Cues:} `Melanie pets family effect', `Melanie pets brighten day' \\
\cmidrule(lr){2-2}

& \textbf{Melanie's cat Bailey addition} \\
& \textit{Value:} Melanie mentions that they have another cat named Bailey. \\
& \footnotesize \textbf{Cues:} `Melanie cat Bailey' \\
\bottomrule
\end{tabular}
\label{tab:retrieval_case_2_1}
\end{table}

\begin{table}[ht]
\centering
\small
\renewcommand{\arraystretch}{1.4}
\caption{Comparison of top retrieved memories for Case Study 2. (continued)}

\begin{tabular}{@{} l p{10.5cm} @{} }
\toprule
\textbf{Method} & \textbf{Retrieved Memories / Context} \\
\midrule

Mem0 & Melanie has kids. \\
& Melanie loves painting animals, especially horses. \\
& User knows Melanie. \\
& Melanie has been painting to keep busy. \\
& Caroline finds joy in having pets. \\
& Melanie paints horses. \\
& Name is Melanie. \\
\midrule
RAG & \textbf{Transcript Fragment:} 3:31 pm on 23 August, 2023 | Caroline: ...And yup, I do— Oscar, my guinea pig. He's been great. How are your pets? \dots 3:31 pm on 23 August, 2023 | Melanie: Yeah, it's normal to be both excited and nervous with a big decision. And thanks for asking, they're good— we got another cat named Bailey too. Here's a pic of Oliver. Can you show me one of Oscar?... \\
\bottomrule
\end{tabular}
\label{tab:retrieval_case_2_2}
\end{table}

\subsection{Case Study 3}

In this example (Table \ref{tab:case3}), a scrutiny of the retrieved evidence reveals that while baseline methods identify the correct topical domain, they fail to capture the discriminative details required for an accurate response. The RAG retrieval is too broad and lacks relevance; it becomes anchored to a dense, irrelevant dialogue fragment about a ``colorful bowl" from a separate project, illustrating how raw context windows are easily distracted by high-signal but incorrect semantic clusters. Meanwhile, Mem0 produces a set of isolated, low-entropy facts, such as ``the kids enjoyed making things with clay", which, while factually true, are too fragmented and generic to support the specific query. By contrast, \textsc{Memora} successfully preserves the fine-grained entity binding between the ``kids' pottery" and the ``dog-face cup." Its index-value architecture prevents the information decay seen in Mem0 and the noise contamination seen in RAG, ensuring that specific attributes remain intact within the retrieved memory.

\begin{table}[h]
\centering
\small
\setlength{\tabcolsep}{6pt}
\caption{Case Study 3 and answers generated from three systems.}
\begin{tabular}{p{3.2cm} p{10cm}}
\toprule
Question & What kind of pot did Mel and her kids make with clay? \\
\midrule
Reference Answer & A cup with a dog face on it \\
\midrule
RAG Answer & \xmark \space A colorful bowl with various colors and patterns \\
Mem0 Answer & \xmark \space Black and white designed bowl \\
\textsc{Memora} Answer & \cmark \space A cup with a dog face \\
\bottomrule
\end{tabular}
\label{tab:case3}
\end{table}

\begin{table}[ht]
\centering
\small
\renewcommand{\arraystretch}{1.5}
\caption{Comparison of top retrieved memories for Case Study 3. (part 1)}

\begin{tabular}{@{} l p{10.5cm} @{} }
\toprule
\textbf{Method} & \textbf{Retrieved Memories / Contextual Evidence} \\
\midrule
\textsc{Memora} & \textbf{Melanie's kids enjoying pottery making} \\
& \textit{Value:} Melanie's kids loved making pottery and were very excited to get their hands dirty and create something with clay. \\
& \footnotesize \textbf{Cues:} `Melanie kids pottery', `Kids clay crafting' \\
\cmidrule(lr){2-2}

& \textbf{Melanie's feelings about clay} \\
& \textit{Value:} Melanie finds clay to be incredible and it brings her a lot of joy. \\
& \footnotesize \textbf{Cues:} `Melanie clay appreciation', `Melanie joy from clay' \\
\cmidrule(lr){2-2}

& \textbf{Melanie's recent activity with kids} \\
& \textit{Value:} Melanie took her kids to a park yesterday, where they had fun exploring and playing. \\
& \footnotesize \textbf{Cues:} `Melanie park visit', `Kids outdoor play' \\
\cmidrule(lr){2-2}

& \textbf{Melanie's kids pottery finished pieces} \\
& \textit{Value:} The kids created pottery finished pieces, including a cup with a dog face on it. \\
& \footnotesize \textbf{Cues:} `Kids pottery artwork', `Pottery cup dog face' \\
\cmidrule(lr){2-2}

& \textbf{Melanie's pottery as a creative and therapeutic outlet} \\
& \textit{Value:} Melanie signed up for a pottery class, which she considers therapeutic and allows her to express herself and be creative. Melanie loves that pottery is both a creative outlet and a form of therapy. Melanie uses pottery as a means for self-expression and to find peace. \\
& \footnotesize \textbf{Cues:} `Melanie pottery therapy', `Pottery creative expression' \\
\bottomrule
\end{tabular}
\label{tab:retrieval_case_3_1}
\end{table}

\begin{table}[t]
\centering
\small
\renewcommand{\arraystretch}{1.4}
\caption{Comparison of top retrieved memories for Case Study 3. (continued)}

\begin{tabular}{@{} l p{10.5cm} @{} }
\toprule
\textbf{Method} & \textbf{Retrieved Memories / Context} \\
\midrule
Mem0 & The kids enjoyed making things with clay. \\
& Melanie took her kids to a pottery workshop last Friday. \\
& Melanie recently finished a pottery project. \\
& Pottery is a huge part of Melanie's life and helps her express her emotions. \\
& Melanie is proud of her pottery project and had a great experience making it. \\
& Enjoyed making pots with kids. \\
& The kids loved making something with clay. \\
\midrule
RAG & \textbf{Transcript Fragment:} 1:50 pm on 17 August, 2023 | Melanie: FYI, I finished another pottery project— want to see a pic? \dots 1:50 pm on 17 August, 2023 | Caroline: That bowl is awesome, Mel! What gave you the idea for all the colors and patterns? \dots 1:50 pm on 17 August, 2023 | Melanie: Thanks, Caroline! I'm obsessed with those, so I made something to catch the eye... \\
\bottomrule
\end{tabular}
\label{tab:retrieval_case_3_2}
\end{table}
\section{Memory Update Analysis}
\label{app:update_analysis}

\begin{table}[t!]
\centering
\caption{Sensitivity to the memory update threshold on LoCoMo.}
\resizebox{0.8\columnwidth}{!}{
\begin{tabular}{l c c c c c c c}
\toprule
Threshold     & Multi & Temp  & Open  & Single & Overall & Avg.\ updates & Update ratio \\
\midrule
No update     & 0.727 & 0.801 & 0.542 & 0.844 & 0.795 &   0 &  0.0\% \\
0.8 (default) & \textbf{0.780} & \textbf{0.813} & 0.510 & 0.836 & \textbf{0.801} &  88 & 21.0\% \\
0.6           & 0.755 & 0.801 & \textbf{0.552} & \textbf{0.842} & 0.799 & 299 & 68.2\% \\
\bottomrule
\end{tabular}
}
\label{tab:update_threshold}
\end{table}

\textbf{Threshold sensitivity.}  \textsc{Memora}'s update mechanism merges a newly extracted memory into an existing entry when their primary abstractions exceed a semantic-similarity threshold. We use a default threshold of 0.80, chosen to merge only entries that are clearly redundant while preventing over-aggressive consolidation.

To isolate the effect of the update threshold from other design choices, we ablate three settings using semantic retrieval with primary abstractions only (no cue anchors or episodic memory); absolute scores are therefore lower than the best results reported in Table~\ref{tab:locomo}.

Table~\ref{tab:update_threshold} shows that the default $0.8$ achieves the highest overall LLM-as-a-Judge score. Lowering the threshold to $0.6$ triggers $3.4\times$ more updates with no quality gain, indicating over-consolidation: semantically distinct entries get merged into the same memory. Disabling updates entirely reduces overhead but hurts quality, particularly on multi-hop questions (0.727 vs.\ 0.780), where consolidation helps resolve redundant or fragmented entries that would otherwise pull the retriever in conflicting directions.

\begin{table}[t!]
\centering
\caption{Update behavior as memories accumulate (threshold $=0.8$).}
\resizebox{0.6\columnwidth}{!}{
\begin{tabular}{l c c c}
\toprule
Memories accumulated & Avg.\ new & Avg.\ update & Update ratio \\
\midrule
0--76    & 77.7 & 16.0 & 17.1\% \\
77--155  & 79.5 & 22.7 & 22.2\% \\
156--234 & 74.9 & 21.4 & 22.2\% \\
235--318 & 82.1 & 23.4 & 22.2\% \\
319+     & 66.4 & 13.1 & 16.5\% \\
\bottomrule
\end{tabular}
}
\label{tab:memory_growth}
\end{table}

\textbf{Memory growth and update.} A natural concern with similarity-based consolidation is whether the update rate grows pathologically as the memory store fills up. On average, each LoCoMo conversation produces 432 candidate memories, of which 344 are inserted as new entries and 88 trigger an update (an overall update ratio of 21.0\%). To examine scaling behavior we bucket each memory-construction event by the number of memories already present at the time of construction; Table~\ref{tab:memory_growth} reports the per-bucket update ratio. The ratio remains in a narrow 16.5--22.2\% band across all buckets rather than increasing exponentially, indicating that consolidation cost scales linearly with memory size and does not become a bottleneck as the memory store grows.

%%%%%%%%%%%%%%%%%%%%%%%%%%%%%%%%%%%%%%%%%%%%%%%%%%%%%%%%%%%%%%%%%%%%%%%%%%%%%%%
%%%%%%%%%%%%%%%%%%%%%%%%%%%%%%%%%%%%%%%%%%%%%%%%%%%%%%%%%%%%%%%%%%%%%%%%%%%%%%%

\end{document}